\documentclass{article}

\PassOptionsToPackage{numbers,square}{natbib}

\usepackage[main, final]{neurips_2026}

\usepackage[utf8]{inputenc} % allow utf-8 input
\usepackage[T1]{fontenc}    % use 8-bit T1 fonts
\usepackage[hidelinks]{hyperref}       % hyperlinks
\usepackage{url}            % simple URL typesetting
\usepackage{booktabs}       % professional-quality tables
\usepackage{amsfonts}       % blackboard math symbols
\usepackage{nicefrac}       % compact symbols for 1/2, etc.
\usepackage{microtype}      % microtypography
\usepackage[table]{xcolor}
\usepackage{graphicx}
\usepackage{eso-pic} % First-page institutional logo for the arXiv edition.
\usepackage{subcaption}
\usepackage{multirow}
\usepackage{wrapfig}
\usepackage{rotating}

\newcommand{\hl}[1]{\cellcolor{red!8}{#1}}
\newcommand{\dec}[1]{\textcolor{blue}{\scriptsize$\downarrow$#1\%}}
\newcommand{\best}[1]{\textcolor{red}{\textbf{#1}}}
\newcommand{\worst}[1]{\textcolor{blue}{#1}}
\newcommand{\inc}[1]{\textcolor{red}{\scriptsize$\uparrow$#1\%}}

\usepackage{amsmath}
\usepackage{amssymb}
\usepackage{mathtools}
\usepackage{amsthm}

\theoremstyle{plain}
\newtheorem{theorem}{Theorem}[section]

\theoremstyle{definition}

\theoremstyle{remark}

\usepackage[textsize=tiny]{todonotes}

\title{Data-Adaptive Mahalanobis Metric Learning for Cross-Head Attention in Transformers}

\author{%
  \textbf{Zhenke Duan}$^{1,2}$ \quad
  \textbf{Xin Li}$^{1}$ \quad
  \textbf{Jiqun Pan}$^{1}$ \\[0.3em]
  \textbf{Dong Xiaofei}$^{1}$ \quad
  \textbf{Hanwen Ning}$^{1,4}$\thanks{Corresponding author.} \quad
  \textbf{Xinyuan Song}$^{3}$ \\[0.8em]
  $^{1}$ School of Statistics and Mathematics, \\
  Zhongnan University of Economics and Law, Wuhan, China \\[0.3em]
  $^{2}$ CodeSoul.Co, Hangzhou, China \\[0.3em]
  $^{3}$ Department of Statistics, \\
  The Chinese University of Hong Kong, Hong Kong SAR, China \\[0.3em]
  $^{4}$ Innovation and Talent Base for Digital Technology and Finance, ZUEL
  \\[0.5em]
}

\begin{document}

% arXiv edition: logo aligns with the text block and does not change pagination.
\AddToShipoutPictureFG*{%
  \AtPageUpperLeft{%
    \hspace*{\dimexpr(\paperwidth-\textwidth)/2\relax}%
    \raisebox{-0.92in}[0pt][0pt]{%
      \includegraphics[width=2.8cm]{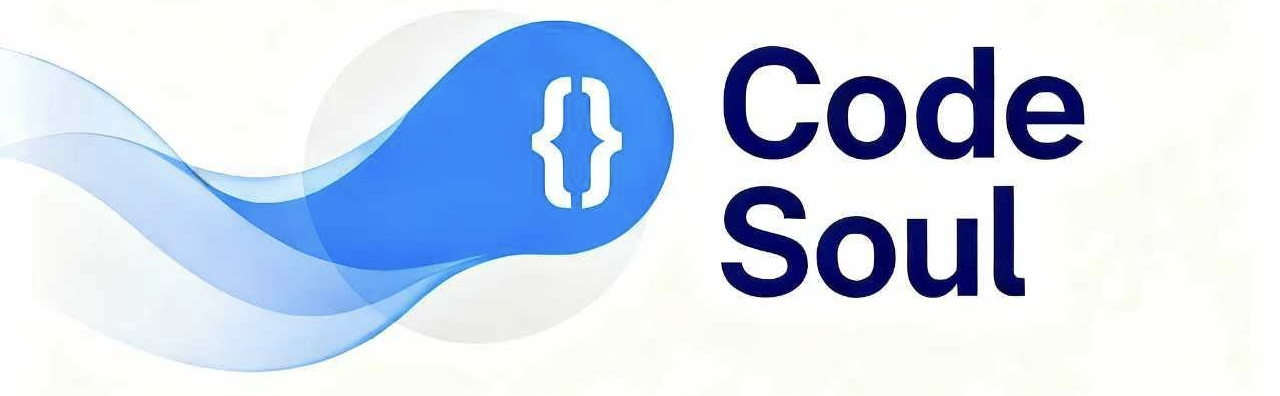}%
    }%
  }%
}

\maketitle

\begin{abstract}
In this paper, we propose Data-Adaptive Mahalanobis Cross-Head Attention (DAMCHA), a simple and effective attention mechanism that generalizes Multi-Head Attention (MHA) through the lens of metric learning. Standard MHA implicitly learns a static, head-wise Euclidean similarity corresponding to the diagonal blocks of a unified cross-head metric matrix, discarding all off-diagonal interactions across heads. DAMCHA addresses this limitation via two key innovations. First, it learns the full cross-head metric matrix, inducing a Mahalanobis-like similarity that explicitly captures inter-head interactions beyond the restricted block-diagonal structure of MHA. Second, it parameterizes the metric matrix as an input-dependent function, enabling the attention geometry to dynamically adapt to the intrinsic manifold structure of the data rather than remaining fixed after training. We further introduce principled regularization strategies, most notably stack-wise parameter sharing,
to ensure computational efficiency and stable optimization.
DAMCHA serves as a drop-in replacement for MHA and its variants in various Transformer architectures, delivering superior expressive power and learning performance with comparable or reduced computational overhead. Extensive experiments demonstrate that DAMCHA-based Transformers consistently outperform competing baselines across a range of benchmark tasks.  Code is available at \href{https://github.com/erwinmsmith/DAMCHA}{https://github.com/erwinmsmith/DAMCHA}.

\end{abstract}

\section{Introduction}

MHA is a core block of Transformer architecture \cite{vaswani2017attention,michel2019sixteen,liu2021multi}. MHA essentially learns a set of similarity matrices induced by a bilinear form (usually denoted as $M_{i,i}=W_Q^{(i)}(W_K^{(i)})^\top$) through a query--key mechanism \cite{ruan2024understanding,shen2022effects,kim2018bilinear}. These learned head-wise metrics are used to compute attention scores, which capture token matching intensities across disparate representation subspaces  \cite{michel2019sixteen,bian2021attention,liu2025multihead}. 
This formulation naturally aligns MHA with metric learning, where discriminative similarity structures are learned to support downstream tasks.\cite{shen2010scalable,liu2020mahalanobis,zhao2025why,zhong2025timevlm,zhang2023a}.

Despite its extensive successes in diverse domains \cite{galassi2021attention,chaudhari2021attentive,khan2022transformers}, when examined through this lens, standard MHA exhibits several inherent structural deficiencies that limit its expressive power. First, it assumes conditional independence between heads for a given input, where similarities are defined solely within each head and merely concatenated at the output. This ignores the cross-head synergies and higher-order interactions corresponding to $M_{i,j}=W_Q^{(i)}(W_K^{(j)})^\top$ for $i \neq j$, making it difficult to represent coupled geometries between complex feature subspaces  \cite{merullo2024talking,cordonnier2021multihead,xiao2024improving}. Second, the similarity metrics $M_{i,i}$ remain static post-training, which implies that, during inference, the attention cannot adapt to the input data, failing to match the non-linear and locally anisotropic manifold structures prevalent in real-world data. In addition, because $M_{i,i}$ relies on a factorized formulation, its rank is inherently capped by the projection dimension, introducing a well-known low-rank bottleneck that limits representation capacity  \cite{zhang2022mixhead,guo2019low,bhojanapalli2020low}. These issues are especially problematic in tasks where the geometry of the data distribution is critical  \cite{karkaria2025asno,yang2024enhancing}.

\par
In this paper, we propose a paradigm shift: rather than viewing the attention as a collection of independent, dot-product, head-wise matrices, it should be understood as a single, learnable Mahalanobis-like metric defined over the joint head space (Here, the term Mahalanobis-like is used to emphasize that our formulation can fully exploit off-diagonal blocks; we do not require the resulting Mahalanobis-like metric matrix to be positive definite). In contrast, standard MHA corresponds to a restrictive block-diagonal constraint that discards all off-diagonal couplings, thereby enforcing an effectively Euclidean similarity measure. From a learning perspective, MHA learns only a global, input-independent similarity matrix that approximates an average attention geometry across the dataset. Guided by this high-level interpretation, we propose a novel attention mechanism, termed DAMCHA. DAMCHA introduces two primary innovations.
\emph{Structural Expansion}: It eliminates structural independence by explicitly learning a full cross-head metric. While its diagonal blocks recover conventional within-head similarities, the off-diagonal blocks encode how different heads should cooperate, compensate, or modulate one another, expanding attention from an additive mixture of isolated similarities to a genuine Mahalanobis-like metric.
 \emph{Dynamic Adaptivity}: Unlike MHA which keeps underlying parameters static, DAMCHA parameterizes the similarity matrix by a learnable function of the input. This allows the model to dynamically select appropriate geometries for different contexts, aligning the attention operator with the local structure of the data manifold. Several prior works have sought to improve upon MHA through head interaction or geometric enrichment, including THA \cite{shazeer2020talking}, DCMHA \cite{xiao2024improving}, ColMHA \cite{cordonnier2020multi}, MoA \cite{zhang2022mixture}, and MMA \cite{cui2019mixed}. However, none of these approaches is grounded in metric learning, and none jointly addresses cross-head interaction and input-adaptive metric learning within a unified framework.

Practically, DAMCHA yields benefits beyond expressiveness: (i) improved modeling of complex dependencies spanning multiple semantic factors (e.g., content--position or syntax--semantics); (ii) enhanced adaptability under distributional shifts or heterogeneous inputs, as the attention geometry reconfigures itself to match the current context; (iii) by directly learning the similarity matrices $M_{i,j}$s through a functional mapping rather than the conventional linear query-key projections, DAMCHA naturally bypasses the low-rank bottleneck. We further introduce principled regularization strategies, most notably stack-wise parameter sharing, to control model complexity. Since the attention matrices in DAMCHA are generated by a compact metric learner, the overall parameter count remains comparable to standard MHA despite the expanded metric structure. This design ensures that the gains in expressiveness come at no significant cost in computational overhead or memory footprint.

Experimental results show that DAMCHA outperforms Transformers across a variety of architectures (e.g. GPT-2) and scales. In language modeling, DAMCHA is particularly superior. Notably, DAMCHA yields improvements in token-level pretraining; these gains are even stronger on mixed datasets, which are inherently cross-view. We further validate DAMCHA on autoregressive, decoder-only Transformers as well as Vision Transformers in generative and discriminative tasks. We provide extensive simulation experiments and ablation studies to elucidate the mechanisms of DAMCHA.

\section{Related Work}

The Transformer relies on self-attention \cite{vaswani2017attention}, where multi-head attention projects queries and keys into parallel subspaces \cite{michel2019sixteen}. Empirical studies show heads often converge to similar or near-identical patterns \cite{bian2021attention,li2023does,jo2025gmar}, so adding heads does not guarantee more diverse representations \cite{michel2019sixteen}. This arises from the block-diagonal metric structure, which excludes cross-head interactions, inducing low-rank bottlenecks \cite{guo2019low,liu2025multihead} and redundant convergence \cite{bian2021attention}.  
Existing remedies, including cross-head interactions or output-level mixing and gating \cite{merullo2024talking,han2025research,xiao2024improving,zhang2022mixhead,lan2025gateformer}, operate post-hoc on attention weights or outputs rather than similarity formation, modifying attention utilization without altering the underlying Euclidean dot-product metric and leaving the core structural limitation unresolved.  

From a metric learning perspective, attention can be interpreted as a parameterized similarity metric \cite{kim2018bilinear}, and existing metric-aware designs focus on domain-specific settings \cite{wang2019deep,liu2020dependency,tay2018latent,yang2025difference}, neglecting multi-head structural features. There is currently no unified framework that integrates Mahalanobis metric learning with multi-head attention to jointly model both intra- and cross-head similarities, leaving open the question of whether similarity metric should adapt dynamically to the input or remain fixed after training \cite{karkaria2025asno,huang2024how}.

\section{DAMCHA}
\label{sec:damcha}
\subsection{Cross-Head Metric Matrix as a Conceptual Formulation}

Let \(X\in\mathbb{R}^{L\times d_{\mathrm{model}}}\) denote an input sequence
of length \(L\), where \(d=d_{\mathrm{model}}=h\cdot d_h\), \(h\) is the number of
attention heads, and \(d_h\) is the hidden dimension of each head. In standard
multi-head attention (MHA), the \(i\)-th head computes query and key
representations as $Q_i=XW_Q^{(i)}$ and $ K_i=XW_K^{(i)}$, where
$W_Q^{(i)},W_K^{(i)}\in\mathbb{R}^{d_{\mathrm{model}}\times d_h}$.
The corresponding pre-softmax score matrix is
\begin{equation}
S_i^{\mathrm{MHA}}(X)
=
\frac{1}{\sqrt{d_h}}Q_iK_i^\top
=
\frac{1}{\sqrt{d_h}}
XW_Q^{(i)}\bigl(W_K^{(i)}\bigr)^\top X^\top .
\label{eq:mha-score-main}
\end{equation}
Thus, from a metric-learning perspective, each head learns an independent
bilinear similarity metric $M_{i,i}=W_Q^{(i)}\bigl(W_K^{(i)}\bigr)^\top $.
Equivalently, standard MHA can be interpreted as a block-diagonal metric
structure over the joint head space: $M_{\mathrm{MHA}}=\operatorname{diag}\bigl(M_{1,1},M_{2,2},\ldots,M_{h,h}\bigr)$.
This block-diagonal structure implicitly assumes that the similarity geometry
of each head is formed independently, without explicitly modeling metric-level
interactions between different heads.
To describe possible interactions beyond this restriction, we introduce the
Cross-Head Metric Matrix (CHM):
\begin{equation}
M =
\begin{pmatrix}
M_{1,1} & \cdots & M_{1,h} \\
\vdots  & \ddots & \vdots  \\
M_{h,1} & \cdots & M_{h,h}
\end{pmatrix},
\label{eq:chm-main}
\end{equation}
where each block \(M_{i,j}\) represents a potential metric-level interaction
between the query subspace associated with head \(i\) and the key subspace
associated with head \(j\). When all off-diagonal blocks vanish, CHM reduces to
the block-diagonal metric structure of standard MHA. When off-diagonal blocks
are allowed, attention is no longer restricted to independent head-wise
similarity measurements.

\subsection{Adaptive Learning of CHM}
\par
Let $X \in \mathbb{R}^{L \times d}$ denote
the input token, and let each metric block satisfy $M_{i,j}(X) \in \mathbb{R}^{d_h \times d_h}$. In MHA, all $M_{i,i}$s are fixed after training, which implies
that the same projection matrices are applied to all inputs $X$ in inference process.
Although attention score $X M_{i,i} X^\top$ vary with the input contents,
the underlying metric $M_{i,i}$ itself is input-independent.
In other words, MHA learns a static metric that is globally shared across all contexts.
This static design may be suboptimal in practice.
For instance, the relevance between two features may be strong in one context but weak in another \cite{xiao2024improving,koyuncu2024eprotran,chen2025dynamic,liu2019attentioncontrolmetriclearning},
yet a fixed metric cannot adapt its attention geometry accordingly.

To increase the expressiveness of the attention mechanism, we parameterize CHM as an input-dependent function. Concretely, a learnable forward network $f_\theta(\cdot)$ (e.g.,
an MLP or a hyper-network) is introduced to map $X$ to an adaptive CHM, which follows
\begin{eqnarray}
f_\theta(X) \rightarrow
M(X) =
\begin{pmatrix}
M_{1,1}(X) & \cdots & M_{1,h}(X) \\
\vdots     & \ddots & \vdots     \\
M_{h,1}(X) & \cdots & M_{h,h}(X)
\end{pmatrix}.
\end{eqnarray}
Under this formulation, each cross-head metric block $M_{i,j}(X)$ is no longer a matrix with constant parameter,
but is dynamically generated conditioned on the input $X$. Further discussions on the advantages of our method are provided in Appendix~\ref{app:Appendix A}.

\subsection{Attention Mechanism based on CHM}
\par
We propose DAMCHA, a novel attention mechanism built upon CHM. We assume that $M(X)$ has been generated by $f_{\theta}(\cdot)$ conditioned on the input sequence. 
DAMCHA consists of $h$ attention heads. For a given input $X$, we first define the \emph{row-summed head metric} for the $i$-th head as
\begin{equation}
\boldsymbol{m}_i(X)
\;=\;
\mathrm{RowSum}(M, i)
\;=\;
\sum_{j=1}^{h} M_{i,j}(X).
\end{equation}
$X$ is first linearly transformed to match the required dimensionality, Then, using $\boldsymbol{m}_i(X)$, DAMCHA computes the score matrix for the $i$-th head via a generalized quadratic form:
\begin{equation}
S_i(X)
\;=\;
\frac{1}{\sqrt{d_k}}\,
X \, \boldsymbol{m}_i(X) \, X^\top
\;\in\;
\mathbb{R}^{L \times L},
\end{equation}
where $i=1,2,\ldots,h$, and $d_k$ is set equal to the per-head dimension $d_h$,
consistent with the scaling in conventional MHA. Then, the attention matrix is obtained by applying a  softmax operation:
\begin{eqnarray}
A_i(X) = \operatorname{Softmax}\big(S_i(X)\big) \in \mathbb{R}^{L \times L}.
\end{eqnarray}
For each head $i \in \{1,\dots,h\}$, we compute value representations as $V_i = X W_V^{(i)}$ with projection
$W_V^{(i)} \in \mathbb{R}^{d \times d_v}$. The output of the $i$-th head is given by
\begin{equation}
O_i(X) = A_i(X)\, V_i \in \mathbb{R}^{L \times d_v},
\end{equation}
where $A_i(X)$ denotes the attention weight and $V_i$ is the corresponding value representation.
The outputs from all heads are concatenated to form a joint output denoted as $O_{\mathrm{cat}}(X)$, given by 
\begin{eqnarray}
O_{\mathrm{cat}}(X)
=
\operatorname{Concat}\bigl(
O_1(X), \ldots, O_h(X)
\bigr).
\end{eqnarray}
Then, $O_{\mathrm{cat}}(X)\in \mathbb{R}^{L \times (h d_v)}$ is mapped back to the original feature dimension, yielding final output $X' = O_{\mathrm{cat}}(X)\, W_O \in \mathbb{R}^{L \times d}$ with projection $W_O \in \mathbb{R}^{(h d_v)\times d}$. 

\begin{wrapfigure}{r}{0.6\columnwidth}
  \centering
  \includegraphics[width=0.6\columnwidth]{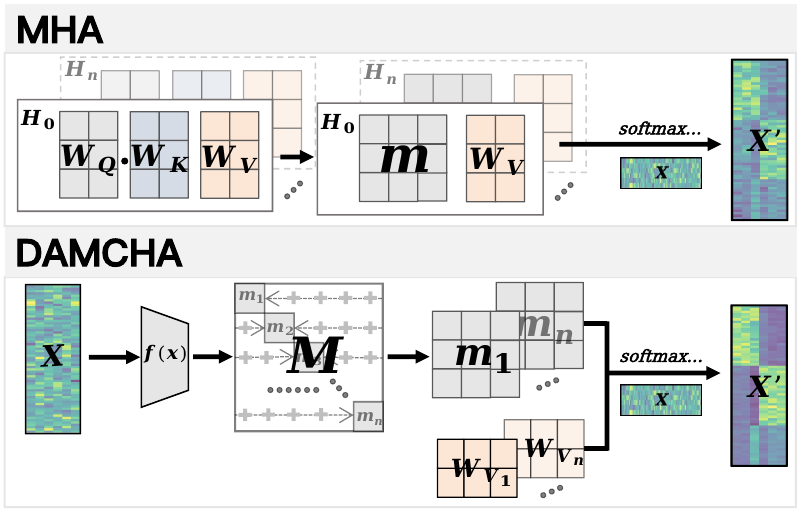}
  \caption{The workflows of DAMCHA and MHA.}
  \label{workflow}
\end{wrapfigure}
The RowSum operation preserves cross-head coupling information while offering two additional advantages. 
First, whereas each head metric $M_{i,i} = W_Q^{(i)}(W_K^{(i)})^\top$ in standard MHA is rank-limited 
by the projection dimension, summing $h$ cross-head blocks elevates the effective rank upper bound, alleviating the low-rank bottleneck. Second, by compressing the full cross-head metric into a single effective matrix per head, RowSum preserves the standard multi-head 
output structure, allowing DAMCHA to serve as a natural drop-in replacement for MHA with negligible additional overhead.
The workflows of DAMCHA and MHA are presented and compared in Figure \ref{workflow}. The data flow of DAMCHA is in Appendix~\ref{app:dataflow}.

\subsection{ Stack-Wise Sharing Mechanism}

In practical applications, attention modules are typically embedded within network blocks that must be stacked to build depth. Learning separate $f_\theta(\cdot)$ mappings (each outputting a different metric matrix for a different block) is not only computationally expensive, but also prone to severe overfitting. To address this issue, we share the parameters of $f_\theta(\cdot)$ across all stacked network blocks.
That is, all layers employ the same feedforward network to generate the adaptive CHMs.
The stack-wise parameter sharing strategy of DAMCHA is further illustrated using a basic decoder-only Transformer architecture, as shown in Figure~\ref{decoder}.

This stack-wise sharing yields multiple benefits.
First, it substantially reduces the total number of parameters, especially when stacking many layers,
thereby mitigating overfitting.
Second, it introduces an additional form of regularization by constraining all layers
to reuse a consistent family of metric transformations.
This prevents excessive divergence of attention geometries across layers
and avoids conflicts between layer-specific metric adaptations.
Empirically, we observe that this shared design leads to more stable training
and improved performance in deep architectures.

It is noted that, in DAMCHA, no explicit structural constraints or prior assumptions are imposed on any of $M_{i,j}(X)$. This is in sharp contrast to MHA, where the similarity matrices associated with individual heads, i.e., $M_{i,i}$, are inherently low-rank due to their dot-product formulation.
By allowing all $M_{i,j}(X)$s to be freely learned without rank restrictions, DAMCHA effectively alleviates the low-rank bottleneck inherent in conventional attention mechanisms, further enhancing the expressive capacity of the model.

\begin{figure}[t]
  \centering
  \includegraphics[width=0.8\textwidth]{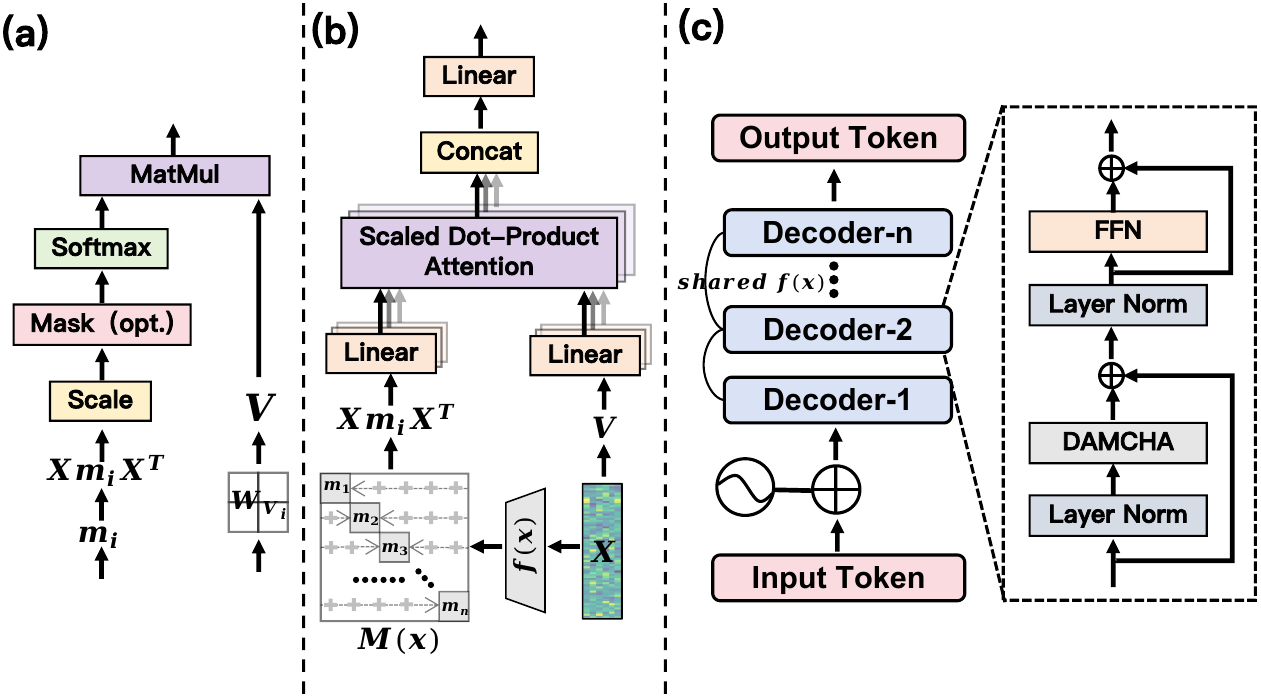}
  \caption{A basic DAMCHA-based decoder-only Transformer architecture with stack-wise sharing.}
  \label{decoder}
\end{figure}

\section{Theoretical View}

We provide two theoretical results to justify DAMCHA, with proofs deferred to 
Appendix~\ref{app:theory}. First, under mild expressivity assumptions on the metric generator, we prove that DAMCHA strictly generalizes standard MHA at the level of pre-softmax attention scores.  
This shows that DAMCHA is not merely a heuristic modification of MHA, but a strict functional extension.
Second, we analyze DAMCHA from a metric-learning risk perspective. In heterogeneous data, different contexts may require different optimal similarity metrics. A static MHA metric must approximate all local geometries using a single global low-rank metric. In contrast, DAMCHA can condition its cross-head metric matrix on the input and select context-specific metrics. Under standard curvature and identifiability assumptions, the optimal risk of DAMCHA is no larger than that of MHA,
with strict improvement whenever no single static low-rank metric can match all context-specific optimal metrics. For a detailed complexity analysis, see Appendix \ref{app:complexity}.

\section{Experiments}
\subsection{Decoder-Only Transformers with DAMCHA}
\label{sec:DAMCHA}

\textbf{Experimental Setting.} We evaluate DAMCHA on both image and text generation tasks using decoder-only Transformers. For image generation, we train autoregressive models on CIFAR-10 and CIFAR-100, where images are flattened into sequences of $32 \times 32 = 1024$ tokens with 8-bit pixel values. For text generation, we use WMT14 \cite{bojar2014findings} English-German translation and CommonGen \cite{lin-etal-2020-commongen} commonsense generation datasets. All models have 10 decoders, 8 attention heads, model dimension 512, and head dimension 64. Models are trained for 50 epochs with batch size 128, learning rate $3\times10^{-4}$, and Adam optimizer. We evaluate the performance using mean square error (MSE) and Fréchet Inception Distance (FID) for images, and Negative Log-Likelihood (NLL) and BARTScore for text. To assess performance on high-quality samples, we additionally report metrics on the top-50 samples ranked by their generation likelihood.
For DAMCHA, the metric generator $f_\theta$ is implemented as a 9-layer MLP,
following the capacity ablation in Section~\ref{sec:ablation}.

Table~\ref{tab:damcha_results_extended} compares DAMCHA with MHA and several recent baselines, THA \cite{shazeer2020talking}, DCMHA \cite{xiao2024improving}, ColMHA \cite{cordonnier2020multi}, MoA \cite{zhang2022mixture}, and MMA \cite{cui2019mixed}. DAMCHA achieves the best quality across both tasks.  

\textbf{Image Generation Performance.}
On CIFAR-10, DAMCHA obtains the lowest MSE mean (0.1210) and FID mean (331.8906), and substantially improves FID@Top-50 from 91.4320 under MHA to 63.0295, corresponding to a 31.1\% relative reduction. Compared with the strongest non-DAMCHA baseline on FID@Top-50, i.e., MMA (84.9717), DAMCHA still achieves a 25.8\% relative reduction. On CIFAR-100, DAMCHA also achieves the best FID mean (314.5722) and FID@Top-50 (53.7560), outperforming both MHA and recent attention variants.

\begin{wrapfigure}{l}{0.5\textwidth}
  \centering
  \includegraphics[width=\linewidth]{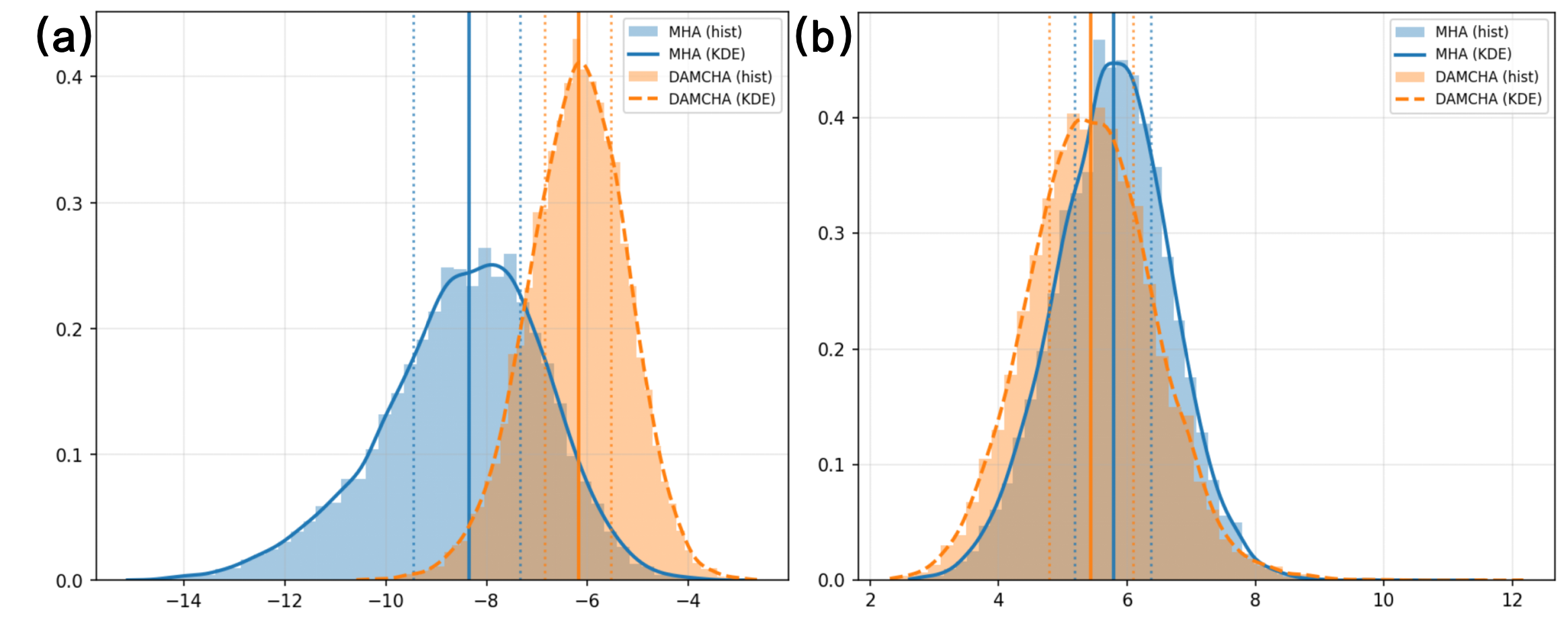}
  \caption{(a) Distribution of BARTScore on WMT data. (b) Distribution of NLL on WMT data.}
  \label{fig:wmt_res}
  \vspace{-0.5em}
\end{wrapfigure}

\begin{table*}[h]
  \centering
  \renewcommand{\arraystretch}{1.08}
  \caption{Comparison with MHA and recent multi-head attention variants on generation tasks.
  For MSE/FID/NLL, lower is better ($\downarrow$). For BARTScore, higher is better ($\uparrow$).}
  \label{tab:damcha_results_extended}
  \resizebox{\textwidth}{!}{%
  \begin{tabular}{llccc ccc c | llccc ccc c}
    \toprule
    \multicolumn{2}{l}{\textbf{Image generation}} &
    \multicolumn{3}{c}{MSE $\downarrow$} &
    \multicolumn{3}{c}{FID $\downarrow$} &
    \multicolumn{1}{c|}{FID@Top-50 $\downarrow$} &
    \multicolumn{2}{l}{\textbf{Text generation}} &
    \multicolumn{3}{c}{NLL $\downarrow$} &
    \multicolumn{3}{c}{BARTScore $\uparrow$} &
    \multicolumn{1}{c}{BARTScore@Top-50 $\uparrow$} \\
    \cmidrule(lr){3-5}\cmidrule(lr){6-8}\cmidrule(lr){12-14}\cmidrule(lr){15-17}
    Dataset & Method & mean & q25 & q75 & mean & q25 & q75 & & Dataset & Method & mean & q25 & q75 & mean & q25 & q75 & \\
    \midrule
 
    \multirow{7}{*}{CIFAR-10}
      & MHA    & 0.1260 & 0.0750 & 0.1610 & 358.8570 & 293.0850 & 417.5810 & 91.4320 &
    \multirow{7}{*}{WMT}
      & MHA    & 5.7730 & 5.1760 & 6.3750 & -8.4740 & -9.4590 & -7.3420 & -5.3090 \\
      & THA    & 0.1357 & 0.0827 & 0.1726 & 366.7536 & 301.3989 & 421.4758 & 107.2156 &
      & THA    & 5.2368 & 4.5104 & 5.7082 & -6.7477 & -7.5275 & -5.8776 & -4.5788 \\
      & DCMHA  & 0.1264 & 0.0758 & 0.1622 & 366.0802 & 297.8571 & 418.4777 & 105.1269 &
      & DCMHA  & 5.2461 & 4.5396 & 5.7308 & -6.8371 & -7.5254 & -6.0639 & -4.8286 \\
      & ColMHA & 0.1276 & 0.0752 & 0.1630 & 365.9343 & 298.6780 & 416.1373 & 120.1717 &
      & ColMHA & 5.2432 & 4.5387 & 5.7324 & -6.1477 & -6.7462 & -5.4902 & -4.3363 \\
      & MMA    & 0.1292 & 0.0771 & 0.1604 & 359.6663 & \textbf{267.3268} & 422.7725 & 84.9717 &
      & MMA    & \textbf{5.0301} & 4.5096 & 5.6847 & -8.3875 & -9.3292 & -7.3488 & -5.7903 \\
      & MoA    & 0.1576 & 0.0988 & 0.2036 & 368.4184 & 297.8040 & 416.2914 & 98.5777 &
      & MoA    & 5.2703 & 4.5695 & 5.7369 & -6.9642 & -7.6396 & -6.2873 & -5.0395 \\
      & \hl{DAMCHA} & \hl{\textbf{0.1210}} & \hl{\textbf{0.0730}} & \hl{\textbf{0.1547}} & \hl{\textbf{331.8906}} & \hl{297.1379} & \hl{\textbf{411.8182}} & \hl{\textbf{63.0295}} &
      & \hl{DAMCHA} & \hl{5.1907} & \hl{\textbf{4.4945}} & \hl{\textbf{5.6793}} & \hl{\textbf{-5.4451}} & \hl{\textbf{-6.0953}} & \hl{\textbf{-5.1735}} & \hl{\textbf{-4.1410}} \\
 
    \midrule
 
    \multirow{7}{*}{CIFAR-100}
      & MHA    & 0.1250 & 0.0737 & 0.1600 & 358.1270 & 293.7460 & 418.5330 & 59.6860 &
    \multirow{7}{*}{CommonGen}
      & MHA    & 5.7560 & 4.6510 & 6.7180 & -5.9250 & -6.4300 & -5.2720 & -4.2520 \\
      & THA    & 0.1352 & 0.0823 & 0.1719 & 364.3847 & 302.4176 & 420.1534 & 57.9292 &
      & THA    & 5.2287 & 4.3996 & 6.0161 & -5.6318 & -6.1313 & \textbf{-4.9114} & -3.7457 \\
      & DCMHA  & 0.1261 & 0.0754 & 0.1618 & 365.2419 & 298.4362 & 419.3781 & 79.9443 &
      & DCMHA  & 5.2754 & 4.3677 & 5.9989 & -5.6686 & \textbf{-5.7876} & -4.9677 & -3.8115 \\
      & ColMHA & 0.1273 & 0.0748 & 0.1625 & 364.5126 & 297.2184 & 417.6439 & 58.1658 &
      & ColMHA & 5.2210 & 4.3567 & 6.0294 & -5.6745 & -5.8728 & -4.9503 & -3.7936 \\
      & MMA    & 0.1289 & 0.0767 & 0.1597 & 358.4291 & 291.4753 & 423.1847 & 55.6137 &
      & MMA    & 5.2952 & \textbf{4.3092} & 5.9861 & -8.0765 & -8.6621 & -7.3931 & -6.2185 \\
      & MoA    & 0.1569 & 0.0981 & 0.2028 & 369.7253 & 298.3416 & 417.4832 & 88.2483 &
      & MoA    & 5.2918 & 4.2990 & 5.9834 & -5.6805 & -5.6899 & -4.9597 & -3.7899 \\
      & \hl{DAMCHA} & \hl{\textbf{0.1232}} & \hl{\textbf{0.0730}} & \hl{\textbf{0.1569}} & \hl{\textbf{314.5722}} & \hl{\textbf{288.9400}} & \hl{\textbf{411.5314}} & \hl{\textbf{53.7560}} &
      & \hl{DAMCHA} & \hl{\textbf{5.2204}} & \hl{4.3105} & \hl{\textbf{5.9788}} & \hl{\textbf{-5.6142}} & \hl{-6.1651} & \hl{-4.9410} & \hl{\textbf{-3.7417}} \\
 
    \bottomrule
  \end{tabular}%
  }
\end{table*}

\textbf{Text Generation Performance.} DAMCHA achieves the best BARTScore-based generation quality on both WMT and CommonGen. On WMT, DAMCHA improves BARTScore from -8.4740 under MHA to -5.4451 and improves BARTScore@Top-50 from -5.3090 to -4.1410. Compared with recent attention variants, DAMCHA also obtains the best BARTScore and BARTScore@Top-50, although MMA achieves a lower NLL mean. This suggests that lower token-level likelihood does not necessarily
translate into better sequence-level generation quality. MMA enforces head diversity through fixed attention masks
(global, local, forward, and backward patterns), which may reduce
per-token prediction uncertainty and thus lower NLL, but its rigid
structural constraints limit the model's ability to form flexible,
context-sensitive attention patterns required for coherent long-range
generation. BARTScore captures this distinction: MMA's forced diversity
yields competitive NLL at the cost of substantially degraded holistic
generation quality (BARTScore mean: $-8.3875$ for MMA 
versus $-5.4451$ for DAMCHA),
precisely the dimension that DAMCHA's adaptive cross-head metric is
designed to improve. On CommonGen, DAMCHA achieves the lowest NLL mean (5.2204), the best BARTScore mean (-5.6142), and the best BARTScore@Top-50 (-3.7417), further confirming the effectiveness of the proposed metric learning mechanism.

\begin{wrapfigure}{r}{0.5\textwidth}
  \centering
  \includegraphics[width=\linewidth]{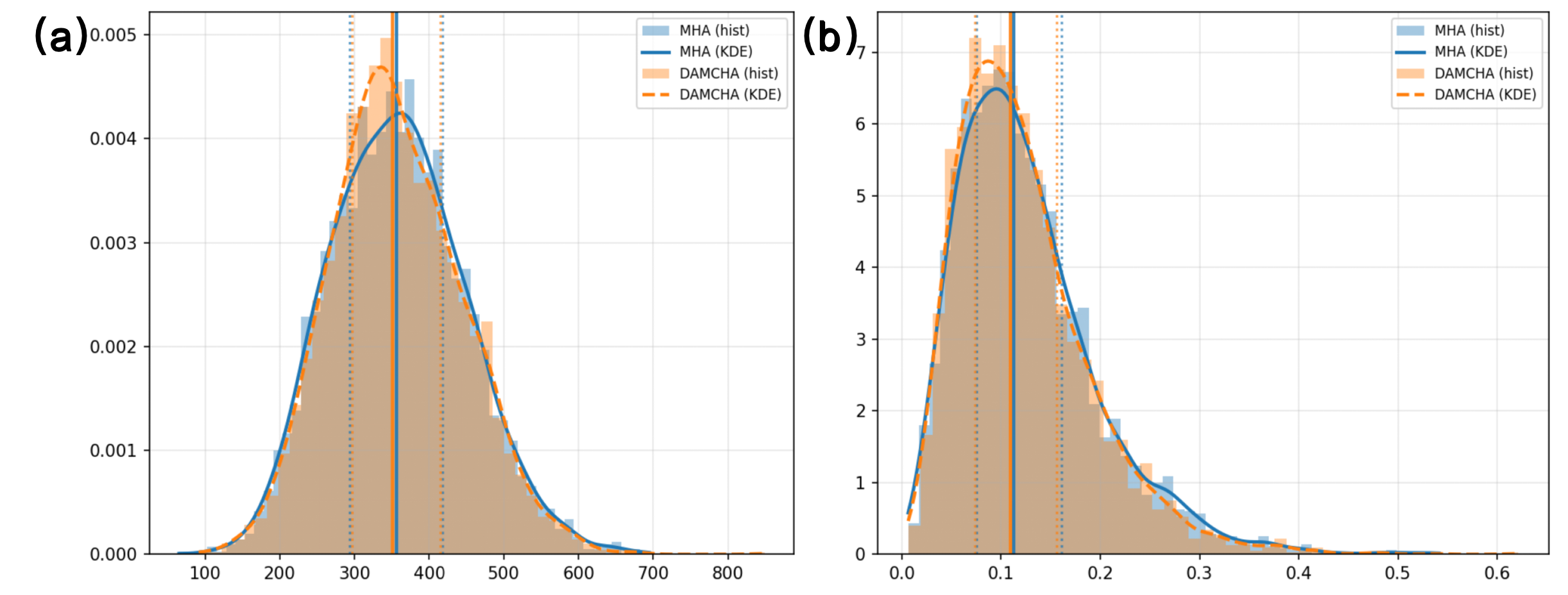}
  \caption{(a) Distribution of FID on CIFAR-10 data. (b) Distribution of MSE on data.CIFAR-10.}
  \label{fig:app_dist_cifar10}
\end{wrapfigure}

\textbf{Performance in High Quality Generation} A critical observation is that the improvement margins on Top-50 metrics are substantially larger than on average metrics. For instance, on WMT, BARTScore@Top-50 improves by 22.0\% ($-$5.3090 to $-$4.1410)
compared to 35.7\% in average BARTScore. The improvements hold consistently across all tasks, suggesting that DAMCHA's input-adaptive metric learning particularly benefits the generation of high-quality, confident outputs where cross-head coordination is critical for capturing complex dependencies.
Figure~\ref{fig:wmt_res} visualizes the distribution shifts on WMT samples. The BARTScore distribution (Figure \ref{fig:wmt_res}(a)) shows a similar rightward shift with reduced variance. The NLL distribution (Figure \ref{fig:wmt_res}(b)) of DAMCHA is notably tighter and shifted leftward compared to MHA, with the mode moving from approximately 5.8 to 5.2.
Figure~\ref{fig:app_dist_cifar10} visualizes the distribution shifts on the CIFAR-10 dataset and further shows that, under DAMCHA, the distributions of both MSE and FID are consistently more favorable than those obtained with MHA.
These distributional changes indicate that DAMCHA not only improves average performance but also reduces uncertainty across the entire sample space, leading to more reliable generation quality. Additional distribution visualizations on CIFAR-100 and CommonGen are provided in Appendix~\ref{app:distplots}. We also visualize the evolution CHM during training in Appendix~\ref{app:Appendix F}.
The results confirms that DAMCHA progressively learns to coordinate attention heads in input-dependent ways, moving beyond the block-diagonal constraint of standard MHA to a full Mahalanobis-like metric over the joint head space.

% \begin{table}[t]
%   \centering
%   \footnotesize
%   \setlength{\tabcolsep}{4pt}
%   \caption{Ablation results. Here, $f(x)$ denotes the generator that produces the diagonal entries of the metric matrix, implemented as a simple MLP.}
%   \label{tab:ablation}

%   \resizebox{\columnwidth}{!}{%
%   \begin{tabular}{llcc}
%     \toprule
%     Dataset & Variant & Metric & Metric@Top-50 \\
%     \midrule
%     \multirow{3}{*}{CIFAR-10 (FID $\downarrow$)}
%       & w/o $f(x)$+share & 233.038 & 76.554 \\
%       & w/o share     & 228.961 & 77.851 \\
%       & DAMCHA        & \textbf{227.641} & \textbf{65.132} \\
%     \midrule
%     \multirow{3}{*}{CIFAR-100 (FID $\downarrow$)}
%       & w/o $f(x)$+share & 229.467 & 61.950 \\
%       & w/o share     & 234.296 & 57.934 \\
%       & DAMCHA        & \textbf{211.290} & \textbf{53.756} \\
%     \midrule
%     \multirow{3}{*}{WMT (BARTScore $\uparrow$)}
%       & w/o $f(x)$+share & -6.868 & -4.972 \\
%       & w/o share     & -6.932 & -4.930 \\
%       & DAMCHA        & \textbf{-6.213} & \textbf{-4.672} \\
%     \midrule
%     \multirow{3}{*}{CommonGen (BARTScore $\uparrow$)}
%       & w/o $f(x)$+share & -7.321 & -4.784 \\
%       & w/o share     & -6.814 & -4.817 \\
%       & DAMCHA        & \textbf{-5.775} & \textbf{-4.196} \\
%     \bottomrule
%   \end{tabular}%
%   }
% \end{table}

\subsection{Efficiency and Computational Cost}
\label{subsec:efficiency}

We further evaluate DAMCHA and the baselines in terms of memory usage, GFLOPs, latency, and throughput. The results are reported in Table~\ref{tab:efficiency_cost}. 

\textbf{Image Generation Tasks.} DAMCHA requires memory usage comparable to ColMHA and lower than THA, DCMHA, and MoA. For CIFAR-10, DAMCHA uses 0.0786 GFLOPs, which are very close to those of the best-performing baseline, MMA, while achieving substantially better FID and FID@Top-50. On CIFAR-100, DAMCHA similarly maintains low memory usage and competitive latency while delivering the best generation quality. Although THA and ColMHA achieve slightly higher throughput on some image settings, their generation metrics are consistently worse than DAMCHA, indicating that DAMCHA provides a better quality--efficiency balance.

\begin{table}[t]
  \centering
  \scriptsize
  \renewcommand{\arraystretch}{1.08}
  \setlength{\tabcolsep}{4pt}
  \caption{Efficiency comparison with recent multi-head attention variants.}
  \label{tab:efficiency_cost}
  \resizebox{\textwidth}{!}{%
  \begin{tabular}{l rrr rrr rrr rrr}
    \toprule
    & \multicolumn{3}{c}{CIFAR-10} & \multicolumn{3}{c}{CIFAR-100} & \multicolumn{3}{c}{CommonGen} & \multicolumn{3}{c}{WMT} \\
    \cmidrule(lr){2-4}\cmidrule(lr){5-7}\cmidrule(lr){8-10}\cmidrule(lr){11-13}
    Method & Mem. & GFLOPs & Lat.(ms) & Mem. & GFLOPs & Lat.(ms) & Mem. & GFLOPs & Lat.(ms) & Mem. & GFLOPs & Lat.(ms) \\
    \midrule
    THA    & 255.40          & 0.0797          & \textbf{3.494}          & 255.40          & 0.0797          & \textbf{3.535}          & 406.80  & 0.1845          & 3.611          & 1416.68 & 0.4698          & 7.299          \\
    DCMHA  & 292.71          & 0.0785          & 4.443          & 292.71          & 0.0785          & 4.499          & 463.36  & 0.1780          & 4.428          & 1650.30 & 0.4399          & 8.148          \\
    ColMHA & 207.39          & 0.0765          & 3.576          & 207.39          & 0.0765          & 3.552 & 282.79  & 0.1699          & \textbf{3.529} &  903.74 & 0.4068          & \textbf{3.868} \\
    MMA    & \textbf{191.39} & \textbf{0.0755} & 16.677         & \textbf{191.39} & \textbf{0.0755} & 17.606         & 281.05  & 0.1678          & 25.863         &  923.74 & 0.4027          & 44.356         \\
    MoA    & 321.07          & 0.1101          & 4.078          & 321.07          & 0.1101          & 4.098          & 484.35  & 0.2538          & 4.095          & 1649.67 & 0.6417          & 6.770          \\
    \rowcolor{red!8}
    DAMCHA & 204.70          & 0.0786          & 3.868          & 196.57          & 0.0786          & 3.843          & \textbf{271.23} & \textbf{0.1552} & 3.844          & \textbf{866.92} & \textbf{0.3586} & 3.895          \\
    \bottomrule
  \end{tabular}%
  }
\end{table}
% \begin{table}[t]
%   \centering
%   \scriptsize
%   \renewcommand{\arraystretch}{1.08}
%   \setlength{\tabcolsep}{4pt}
%   \caption{Efficiency comparison with recent multi-head attention variants.}
%   \label{tab:efficiency_cost}
%   \resizebox{\textwidth}{!}{%
%   \begin{tabular}{ll rr rr rr rr}
%     \toprule
%     & & \multicolumn{2}{c}{CIFAR-10} & \multicolumn{2}{c}{CIFAR-100} & \multicolumn{2}{c}{CommonGen} & \multicolumn{2}{c}{WMT} \\
%     \cmidrule(lr){3-4}\cmidrule(lr){5-6}\cmidrule(lr){7-8}\cmidrule(lr){9-10}
%     Method & & Memory & GFLOPs & Memory & GFLOPs & Memory & GFLOPs & Memory & GFLOPs \\
%     \midrule
%     THA    & & 255.40 & 0.0797 & 255.40 & 0.0797 & 406.80 & 0.1845 & 1416.68 & 0.4698 \\
%     DCMHA  & & 292.71 & 0.0785 & 292.71 & 0.0785 & 463.36 & 0.1780 & 1650.30 & 0.4399 \\
%     ColMHA & & 207.39 & 0.0765 & 207.39 & 0.0765 & 282.79 & 0.1699 &  903.74 & 0.4068 \\
%     MMA    & & \textbf{191.39} & \textbf{0.0755} & \textbf{191.39} & \textbf{0.0755} & 281.05 & 0.1678 &  923.74 & 0.4027 \\
%     MoA    & & 321.07 & 0.1101 & 321.07 & 0.1101 & 484.35 & 0.2538 & 1649.67 & 0.6417 \\
%     \hl{DAMCHA} & & \hl{204.70} & \hl{0.0786} & \hl{196.57} & \hl{0.0786} & \hl{\textbf{271.23}} & \hl{\textbf{0.1552}} & \hl{\textbf{866.92}} & \hl{\textbf{0.3586}} \\
%     \bottomrule
%   \end{tabular}%
%   }
% \end{table}

\textbf{Text Generation Tasks.} The efficiency advantage becomes more pronounced on text generation tasks. On WMT, DAMCHA uses 866.92 memory units and 0.3586 GFLOPs, which are lower than THA, DCMHA, ColMHA, MMA, and MoA. It also achieves much lower latency than THA, DCMHA, MMA, and MoA. On CommonGen, DAMCHA also uses substantially less memory than THA, DCMHA, and MoA while maintaining competitive latency and throughput. These results show that DAMCHA's improvements are not simply obtained by increasing model size or computation, but are mainly attributable to the proposed cross-head metric learning design.

\begin{table*}[t]
  \centering
  \small
  \setlength{\tabcolsep}{4pt}
  \caption{Ablation on the capacity of the standard DAMCHA's $f(x)$ at different depth, and inference using full-, diagonal-only, and off-diagonal-only metric matrices. Within each row (fixed setting and dataset), the best layer is highlighted in \textcolor{red}{red bold} and the worst layer in \textcolor{blue}{blue} for each metric block.}
  \label{tab:damcha_diag_offdiag_ablation}

  \resizebox{\textwidth}{!}{%
  \begin{tabular}{llcccccccccc}
    \toprule
    \multirow{2}{*}{Setting} & \multirow{2}{*}{Dataset} &
    \multicolumn{5}{c}{FID@Top-50 / BART@Top-50} &
    \multicolumn{5}{c}{FID / BART} \\
    \cmidrule(lr){3-7}\cmidrule(lr){8-12}
    & & 1-layer & 3-layer & 5-layer & 7-layer & 9-layer
      & 1-layer & 3-layer & 5-layer & 7-layer & 9-layer \\
    \midrule

    \multirow{4}{*}{DAMCHA}
      & CIFAR-10
      & \worst{77.97} & 75.79 & 71.07 & 74.07 & \best{58.64}
      & \worst{250.35} & 231.52 & 234.37 & \best{229.09} & 229.84 \\
      & CIFAR-100
      & \worst{65.20} & 60.82 & 57.64 & 52.14 & \best{50.29}
      & \worst{275.80} & 240.98 & 221.85 & 237.57 & \best{217.89} \\
      & WMT
      & \worst{-5.26} & -5.14 & -5.08 & -4.98 & \best{-4.51}
      & \worst{-7.71}  & -7.27  & -6.00  & -6.32  & \best{-5.61} \\
      & CommonGen
      & \worst{-4.97} & -4.90 & -4.62 & -4.84 & \best{-4.60}
      & \worst{-7.83}  & -6.53  & -7.65  & -6.90  & \best{-6.20} \\
    \midrule

    \multirow{4}{*}{Diagonal-only}
      & CIFAR-10
      & 102.51 & 105.05 & 108.505 & 98.703 & 112.856
      & 248.38 & 233.05 & 235.17 & 233.485 & 239.609 \\
      & CIFAR-100
      & 85.763 & 69.73  & 76.102 & 59.685 & 61.945
      & 275.535 & 227.707 & 236.587 & 231.397 & 237.241 \\
      & WMT
      & -6.218 & -6.043 & -5.622 & -5.457 & -5.618
      & -7.186 & -7.411 & -7.070 & -7.854 & -7.859 \\
      & CommonGen
      & -6.068 & -5.213 & -5.336 & -4.865 & -5.557
      & -7.674 & -6.770 & -7.612 & -7.039 & -6.956 \\
    \midrule

    \multirow{4}{*}{Off-diagonal-only}
      & CIFAR-10
      & 97.318 & 140.37 & 87.217 & 97.647 & 126.921
      & 245.87 & 252.75 & 250.201 & 249.907 & 241.249 \\
      & CIFAR-100
      & 91.10 & 104.128 & 105.026 & 97.144 & 72.746
      & 268.725 & 246.883 & 247.715 & 250.512 & 236.142 \\
      & WMT
      & -6.892 & -6.101 & -6.335 & -5.997 & -6.488
      & -8.095 & -7.140 & -8.112 & -6.957 & -8.119 \\
      & CommonGen
      & -5.10 & -5.681 & -5.954 & -5.2457 & -5.854
      & -6.926 & -8.626 & -8.913 & -7.154 & -8.467 \\
    \bottomrule
  \end{tabular}%
  }
\end{table*}

\subsection{Ablation Study}
\label{sec:ablation}

\textbf{Metric Learning \& Stack-Wise Sharing.} We first ablate the two key innovations of DAMCHA: (1) the input-adaptive metric generator $f(x)$ and (2) the stack-wise sharing mechanism. Table~\ref{tab:ablation} presents results under three configurations: \textit{w/o $f(x)$+share} removes both components, reducing DAMCHA to a static, unshared cross-head metric per-layer; \textit{w/o share} retains input-adaptivity but uses layer-specific generators; and DAMCHA with both mechanisms.

The results demonstrate that both components are critical. Removing both $f(x)$ and sharing leads to the largest degradation---on CIFAR-10, FID@Top-50 increases by 17.59\% (from 65.132 to 76.554), while average FID degrades by only 2.37\%. This substantial gap confirms that input-adaptive metric learning is particularly crucial for high-quality generation. The \textit{w/o share} variant shows moderate degradation, with FID@Top-50 increasing by 19.59\% on CIFAR-10, indicating that independent per-layer causes overfitting and instability.
Across all datasets, the pattern is consistent: removing input-adaptivity causes larger performance drops on Top-50 metrics than on average metrics. On CommonGen, BARTScore@Top-50 degrades by 13.99\% without $f(x)$ compared to only 26.77\% degradation in average BARTScore. Moreover, we provide a visualized gradient-flow analysis in Appendix~\ref{app:Appendix G} to illustrate how our stack-wise sharing improves learning performance.

\begin{wraptable}{l}{0.5\textwidth}
  \centering
  \scriptsize
  \setlength{\tabcolsep}{3pt}
  \caption{Ablation results. Here, $f(x)$ is implemented as an MLP. Blue $\downarrow$ indicates the relative degradation w.r.t.\ DAMCHA (in \%).}
  \label{tab:ablation}
  \resizebox{\linewidth}{!}{%
  \begin{tabular}{llcc}
    \toprule
    Dataset & Variant & Metric & Metric@Top-50 \\
    \midrule
    \multirow{3}{*}{CIFAR-10 (FID $\downarrow$)}
      & w/o $f(x)$+share & 233.038\,\dec{-2.37}  & 76.554\,\dec{-17.59} \\
      & w/o share        & 228.961\,\dec{-0.58}  & 77.851\,\dec{-19.59} \\
      & DAMCHA           & \textbf{227.641}       & \textbf{65.132} \\
    \midrule
    \multirow{3}{*}{CIFAR-100 (FID $\downarrow$)}
      & w/o $f(x)$+share & 229.467\,\dec{-8.66}  & 61.950\,\dec{-15.25} \\
      & w/o share        & 234.296\,\dec{-10.95} & 57.934\,\dec{-7.77} \\
      & DAMCHA           & \textbf{211.290}       & \textbf{53.756} \\
    \midrule
    \multirow{3}{*}{WMT (BARTScore $\uparrow$)}
      & w/o $f(x)$+share & -6.868\,\dec{-10.54}  & -4.972\,\dec{-6.43} \\
      & w/o share        & -6.932\,\dec{-11.57}  & -4.930\,\dec{-5.54} \\
      & DAMCHA           & \textbf{-6.213}        & \textbf{-4.672} \\
    \midrule
    \multirow{3}{*}{CommonGen (BARTScore $\uparrow$)}
      & w/o $f(x)$+share & -7.321\,\dec{-26.77}  & -4.784\,\dec{-13.99} \\
      & w/o share        & -6.814\,\dec{-17.99}  & -4.817\,\dec{-14.78} \\
      & DAMCHA           & \textbf{-5.775}        & \textbf{-4.196} \\
    \bottomrule
  \end{tabular}%
  }
  \vspace{-0.5em}
\end{wraptable}

\textbf{Depth-wise Metric Learning \& Cross-Head Learning.} To analyze how metric learning strength affects learning performance and to verify the respective contributions of diagonal and off-diagonal blocks, we conduct three experiments: (1) standard DAMCHA using the full metric matrix, (2) diagonal-only inference where we zero out all off-diagonal blocks $M_{i,j}$ for $i \neq j$, and (3) off-diagonal-only inference where we zero out all diagonal blocks. For each setting, we evaluate the performance of $f(x)$ with different capacities (with 1, 3, 5, 7, 9 hidden layers).

Table~\ref{tab:damcha_diag_offdiag_ablation} reveals two key findings. First, deeper layers consistently achieve better performance under the standard DAMCHA setting. For example, on CIFAR-10, FID@Top-50 improves from 77.97 at layer~1 to 58.64 at layer~9, while the average FID decreases from 250.35 to 229.84. This trend is consistent across all datasets, confirming that later layers learn increasingly task-relevant metric transformations. This hierarchical pattern aligns with the metric matrix evolution shown in Figure~\ref{fig:cifar10_m_matrix}, where deeper layers develop more specialized interaction structures. These results indicate that appropriately designed metric learning is necessary for enhancing the expressive power of model.

Second, the diagonal-only inference results indicate that the block-diagonal constraint imposed by standard MHA is indeed limiting. Compared with full DAMCHA, diagonal-only inference exhibits a substantial performance degradation, particularly on the Top-50 metrics. On CIFAR-10, even the best diagonal-only configuration (layer~7, FID@Top-50 = 98.703) performs worse than the weakest full DAMCHA configuration (layer~1, FID@Top-50 = 77.97). Moreover, off-diagonal-only inference performs significantly worse than both full and diagonal-only metrics. This observation suggests that the data inherently contains interactions at cross-head level, and that cross-head interactions modeled through off-diagonal blocks are essential for capturing complex dependencies. More visualized results on metric geometry in ablations are presented in Appendix~\ref{app:ablationviz}.

\subsection{Autoregressive Generation with GPT-2}

We apply DAMCHA to the GPT-2 and conduct token-level pretraining experiments on three datasets:
Enwik8 (100M characters of Wikipedia text available at \url{https://huggingface.co/google/reformer-enwik8}.),
Text8 (100M characters of cleaned Wikipedia, available at \url{https://huggingface.co/datasets/afmck/text8}), and Mix (200M characters, a 50-50 mixture of Enwik8 and Text8). All models have 12 layers, 12 attention heads, model dimension 768, and are trained for 100K steps with batch size 64, context length 512, and learning rate $6\times10^{-4}$. We report cross-entropy loss and perplexity (PPL) on held-out validation sets. Since DCMHA achieves the best performance among all baselines, we report only the comparison results against DCMHA and MHA. 
\par
Table~\ref{tab:mha_vs_damcha_token} 
demonstrates that DAMCHA consistently outperforms both MHA and DCMHA across all three datasets. Notably, the gains are largest on the Mix dataset, where DAMCHA achieves 5.19\% and 3.70\% loss
reduction over MHA and DCMHA respectively. This is consistent with DAMCHA's design rationale:
the Mix dataset's heterogeneous semantic patterns demand cross-head coordination that MHA's
block-diagonal structure cannot provide and DCMHA only partially addresses, while DAMCHA
dynamically reconfigures its attention geometry to match the local context. Since DCMHA achieves the best performance among all baselines in the main generation experiments (Table \ref{tab:damcha_results_extended}), we focus our comparison on MHA and DCMHA to highlight the most informative contrast; the remaining baselines follow the same relative ordering observed in Section \ref{sec:DAMCHA}.

\subsection{Discriminative Evaluation with Vision Transformers}
\begin{wraptable}{r}{0.6\textwidth}
  \centering
  \scriptsize
  \setlength{\tabcolsep}{4pt}
  \caption{Token-level pretraining results. Lower is better for cross-entropy loss and PPL. Red $\uparrow$ denotes DAMCHA's relative improvement over DCMHA (in \%).}
  \label{tab:mha_vs_damcha_token}
  \begin{tabular}{lcccccc}
    \toprule
    Dataset & \multicolumn{3}{c}{Cross-entropy loss $\downarrow$} & \multicolumn{3}{c}{PPL $\downarrow$} \\
    \cmidrule(lr){2-4}\cmidrule(lr){5-7}
            & MHA & DCMHA & DAMCHA & MHA & DCMHA & DAMCHA \\
    \midrule
    Enwik8 & 1.07 & 1.05 & \textbf{1.03}\,\inc{1.90} & 2.90 & 2.86 & \textbf{2.80}\,\inc{2.10} \\
    Text8  & 1.06 & 1.05 & \textbf{1.03}\,\inc{1.90} & 2.88 & 2.86 & \textbf{2.81}\,\inc{1.75} \\
    Mix    & 1.09 & 1.08 & \textbf{1.04}\,\inc{3.70} & 2.98 & 2.94 & \textbf{2.82}\,\inc{4.08} \\
    \bottomrule
  \end{tabular}
   \vspace{-0.5em}
\end{wraptable}

We evaluate DAMCHA on ImageNet-1K classification using ViT-S/16 (12 layers, 6 heads, dimension 384, 22M parameters). Models are trained for 500 epochs with batch size 1024, learning rate 0.001, and standard augmentation. We report Top-1 and Top-5 accuracy at epochs 100 and 500 \cite{deng2009imagenet}.
Table~\ref{tab:vit_results} shows that DAMCHA consistently outperforms both MHA and DCMHA baselines. 
DAMCHA achieves the highest accuracy across all evaluation points. At epoch 100, it obtains 70.93\% Top-1 and 89.71\% Top-5 accuracy, outperforming both MHA and DCMHA and indicating faster early-stage convergence. At epoch 500, DAMCHA further improves to 82.10\% Top-1 and 96.23\% Top-5 accuracy, maintaining consistent advantages over MHA and DCMHA. Although DCMHA narrows the gap after longer training, DAMCHA remains the best-performing model on all reported metrics.
The larger early gap between DAMCHA and DCMHA suggests that DAMCHA's dynamic head coordination enables more efficient feature extraction from the start of training. While DCMHA narrows the gap with extended training, DAMCHA maintains advantages in both convergence speed and final performance across all metrics. Following the focused comparison strategy of Section \ref{sec:DAMCHA}, we report results against MHA and DCMHA, the strongest baseline across the generation experiments.
\begin{wraptable}{l}{0.5\textwidth}
  \caption{Results on ImageNet-1K classification.}
  \label{tab:vit_results}
  \resizebox{\linewidth}{!}{%
  \begin{tabular}{lcccc}
    \toprule
    \multirow{2}{*}{Model} & \multicolumn{2}{c}{Top-1} & \multicolumn{2}{c}{Top-5} \\
    \cmidrule(lr){2-3} \cmidrule(lr){4-5}
    & epoch 100 & epoch 500 & epoch 100 & epoch 500 \\
    \midrule
    MHA     & 69.02\% & 79.51\% & 88.68\% & 94.68\% \\
    DCMHA  & 67.90\% & 81.70\% & 88.71\% & 95.00\% \\
    DAMCHA  & \textbf{70.93\%} & \textbf{82.10\%} & \textbf{89.71\%} & \textbf{96.23\%} \\
    \bottomrule
  \end{tabular}%
  } \vspace{-0.5em}
\end{wraptable}
\subsection{Experiments with Synthetic Tasks}
To further understand the mechanism underlying DAMCHA's improvements, we design two synthetic
tasks that explicitly require coordinated attention across multiple
subspaces. These controlled experiments reveal that DAMCHA's advantage grows with the strength of
cross-subspace entanglement: DAMCHA's relative gain increases and scales
with feature dimensionality in the classification task. Crucially, when cross-subspace interactions
are weak or absent, DAMCHA remains competitive with MHA, confirming that the adaptive cross-head
metric incurs no penalty in simpler settings. These results provide controlled evidence that
DAMCHA's improvements on real-world benchmarks are driven by its ability to model dependencies that
span multiple attention heads, rather than by an increase in parameter count or computational cost.
Full experimental details and results are provided in Appendix~\ref{app:synthetic}. Computational resource details are reported in
Appendix~\ref{app:compute}, and a mechanistic
discussion of DAMCHA's advantages over competing baselines is provided in
Appendix~\ref{app:comparison}.

\section{Conclusion}
This paper proposes a novel attention mechanism, DAMCHA. By lifting both the block-diagonal and input-independent constraints of MHA, DAMCHA generalizes it into a data-adaptive metric learning framework. This formulation enables dynamic reconfiguration of attention geometry and explicit cross-head interaction modeling, providing significantly enhanced representational and approximation capabilities.
Extensive experiments consistently demonstrate that DAMCHA delivers stronger performance, with particularly pronounced gains on high-quality and high-confidence samples. These results confirm that learning input-adaptive cross-head metrics for attention effectively enhances the representational capacity of Transformers, making DAMCHA well suited for complex dependency modeling and generative tasks.

\setcitestyle{square,numbers}
\bibliographystyle{elsarticle-num}
\bibliography{references}
%%%%%%%%%%%%%%%%%%%%%%%

\newpage

\appendix
\section*{Appendix}

\section{Further Illustration on Expressive Power and Advantages of CHM}
\label{app:Appendix A}

The introduction of cross-head interactions significantly enhances representational flexibility.
For example, when the similarity between two tokens depends on interactions between features residing in different subspaces, standard MHA may fail to capture such dependencies if those features are assigned to separate heads.
The CHM directly models these interactions via the off-diagonal blocks $M_{i,j}$, enabling explicit cross-subspace similarity estimation.
From a theoretical standpoint, removing the block-diagonal constraint allows the attention score matrix to represent higher-rank correlation patterns.
As a result, a single attention layer equipped with cross-head metrics can express information structures that would otherwise require multiple stacked layers in standard MHA, leading to more efficient and powerful representation learning.

For our adaptive learning of CHM, the formulation can be interpreted as introducing a form of conditional metric learning, where the attention geometry is modulated by the current context.
Rather than learning a single global notion of similarity,
the model learns how to select or generate an appropriate metric given the input.
Functionally, this resembles dynamic parameterization or conditional computation,
in which a lightweight network produces part of the parameters of the main model
to enable context-dependent behavior.
From a representational standpoint, generating $M(X)$ via $f(\cdot)$ explicitly introduces nonlinear,
higher-order interactions into attention.
The resulting attention score depends not only on the quadratic form $X^\top M X$,
but also on the nonlinear transformation induced by $f$.
As such, the attention mechanism can be viewed as operating on higher-order polynomial features of $X$,
substantially enriching its function approximation and expressive capacity. 

Several advantages are expected to be provided by our method
Context-sensitive attention:
    The model can dynamically adjust its attention patterns across different inputs,
    improving sensitivity to contextual variations.
Enhanced expressiveness:
    By allowing $M$ to depend on $X$, the attention mechanism is no longer restricted to a fixed bilinear form,
    but instead realizes a composite mapping,
    which spans a significantly richer function class and can capture more complex dependencies.
Reduced approximation bias:
    The adaptive mechanism alleviates the bias introduced by static parameters.
    Even if certain relationships are not learned as dominant global patterns during training,
    the model can still approximate them for specific inputs by generating specialized attention metrics,
    thereby reducing representational blind spots.

\paragraph{Asymmetric Design of Value Projection.}
A natural question concerns the asymmetric treatment of the query--key pathway versus the value pathway in DAMCHA: while $f_\theta$ replaces conventional linear query--key projections with an input-adaptive metric generator, the value projection $W_V^{(i)}$ retains the standard static linear form of MHA. This design is deliberate. The query--key pathway determines \emph{where} attention is directed, that is, the similarity geometry, which varies most across contexts and benefits most from input adaptivity. The value projection, by contrast, determines \emph{what} information is aggregated given fixed attention weights, a role that is well served by a static linear transformation in practice. Extending adaptivity to the value pathway would add computational overhead and optimization complexity without the same theoretical motivation: the expressiveness gains in Theorems~\ref{thm:strict_generalization} and~\ref{thm:metric_advantage} arise specifically from the adaptive similarity structure in the pre-softmax score, not from adaptive value representations. Exploring adaptive value projections in combination with DAMCHA's metric learning is nonetheless an interesting direction for future work.

\section{Data Flow and Tensor Shapes of DAMCHA}
\label{app:dataflow}

Figure~\ref{fig:app_dataflow} summarizes the end-to-end data flow of DAMCHA and MHA.

\begin{figure}[h]
    \centering
    \includegraphics[width=\textwidth]{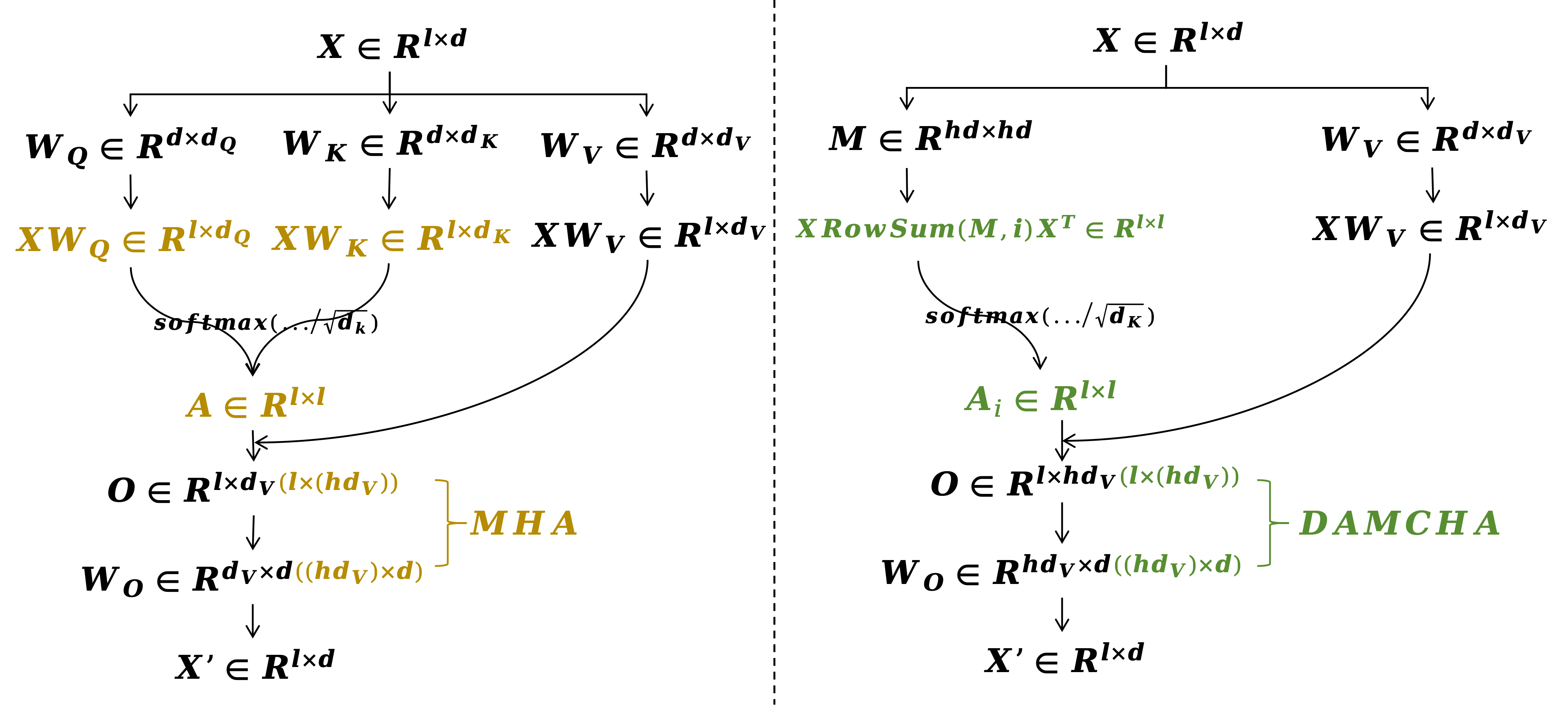}
    \caption{Data flow and tensor shapes of DAMCHA (right, with the value/output projection shapes consistent with standard MHA)}
    \label{fig:app_dataflow}
\end{figure}

\section{Theoretical Analysis}
\label{app:theory}

In this section, we provide a theoretical justification for DAMCHA from two
complementary perspectives. First, we show that the score-function class induced
by DAMCHA strictly contains that of standard MHA. Second, we interpret both MHA
and DAMCHA through the lens of metric learning and prove that input-adaptive
cross-head metrics can reduce the irreducible metric mismatch incurred by a
static global metric.

\subsection{DAMCHA Strictly Generalizes MHA}
\label{subsec:theory_generalization}

We first formalize the attention score functions induced by standard MHA and
DAMCHA. Let \(X\in\mathbb{R}^{L\times d}\) denote an input sequence, where
\(L\) is the sequence length and \(d=d_{\mathrm{model}}\). For the \(i\)-th
head of standard MHA, the pre-softmax score matrix can be written as
\begin{eqnarray}
S_i^{\mathrm{MHA}}(X)
=
\frac{1}{\sqrt{d_h}}X M_i X^\top,
\qquad
M_i = W_Q^{(i)}(W_K^{(i)})^\top ,
\end{eqnarray}
where \(W_Q^{(i)},W_K^{(i)}\in\mathbb{R}^{d\times d_h}\).
Thus, each head in MHA corresponds to a static low-rank bilinear similarity
metric and $\operatorname{rank}(M_i)\le d_h$.
For DAMCHA, the \(i\)-th head is given by
\begin{eqnarray}
\bar M_i(X)=\sum_{j=1}^h M_{i,j}(X),
\end{eqnarray}
where \(M_{i,j}(X)\) is generated by the input-dependent metric generator
\(f_\theta(\cdot)\). The corresponding pre-softmax score matrix is
\begin{eqnarray}
S_i^{\mathrm{DAMCHA}}(X)=\frac{1}{\sqrt{d_k}}X\bar M_i(X)X^\top.
\end{eqnarray}

We define the score-function class induced by standard MHA as
\begin{eqnarray}
\mathcal{S}_{\mathrm{MHA}}
=
\left\{
X\mapsto
\left(
XM_1X^\top,\ldots,XM_hX^\top
\right):
M_i=A_iB_i^\top,\ 
A_i,B_i\in\mathbb{R}^{d\times d_h}
\right\}.
\end{eqnarray}
Similarly, the score-function class induced by DAMCHA is defined as
\begin{eqnarray}
\mathcal{S}_{\mathrm{DAMCHA}}
=
\left\{
X\mapsto
\left(
X\bar M_1(X)X^\top,\ldots,X\bar M_h(X)X^\top
\right):
\bar M_i(X)=\sum_{j=1}^hM_{i,j}(X)
\right\}.
\end{eqnarray}

We impose the following mild expressivity assumption on the metric generator.

\paragraph{Assumption 1.}
The metric generator \(f_\theta\) is capable of representing constant matrix
mappings. That is, for any collection of matrices
\(\{B_{i,j}\}_{i,j=1}^h\), there exists a parameter \(\theta\) such that
$M_{i,j}(X)\equiv B_{i,j}$, for all \(X\).
This assumption is mild because a standard feedforward network or hypernetwork
with bias terms can represent constant functions.

\begin{theorem}[Strict Function-Class Generalization of MHA]
\label{thm:strict_generalization}
Suppose Assumption~1 holds and \(d_h<d\). Then the score-function class induced
by standard MHA is strictly contained in that induced by DAMCHA:
\[
\mathcal{S}_{\mathrm{MHA}}
\subsetneq
\mathcal{S}_{\mathrm{DAMCHA}}.
\]
Consequently, DAMCHA strictly generalizes MHA at the level of pre-softmax
attention score functions.
\end{theorem}

\begin{proof}
We prove the result in two steps. \textbf{Step 1: Inclusion.}
We first show that $\mathcal{S}_{\mathrm{MHA}} \subseteq \mathcal{S}_{\mathrm{DAMCHA}}.$
By Assumption~1, DAMCHA can choose its generated cross-head metric blocks as
constant matrices satisfying $M_{i,i}(X)\equiv M_i$ and $M_{i,j}(X)\equiv 0$, $j\neq i$.
Then the row-summed effective metric of the \(i\)-th head becomes
\begin{eqnarray}
\bar M_i(X)=\sum_{j=1}^h M_{i,j}(X)=M_{i,i}(X)=M_i.
\end{eqnarray}
It follows that
\begin{eqnarray}
X\bar M_i(X)X^\top
=XM_iX^\top.
\end{eqnarray}
Then, every score function produced by MHA can be exactly reproduced by DAMCHA.
This proves
\begin{eqnarray}
\mathcal{S}_{\mathrm{MHA}}
\subseteq
\mathcal{S}_{\mathrm{DAMCHA}}.
\end{eqnarray}

\paragraph{Step 2: Strictness.}
We now show that the inclusion is strict. Since \(d_h<d\), there exists a matrix
\(B\in\mathbb{R}^{d\times d}\) such that $\operatorname{rank}(B)>d_h$.
By Assumption~1, DAMCHA can choose a constant effective metric for the first
head as
\begin{eqnarray}
\bar M_1(X)\equiv B.
\end{eqnarray}
For instance, this can be achieved by setting $M_{1,1}(X)\equiv B$, $M_{1,j}(X)\equiv 0$, $j\neq 1$.
Then DAMCHA induces the score function $X\mapsto XB X^\top.$
Suppose, for contradiction, that this score function can also be represented by
standard MHA. Then there must exist matrices
\(A,B'\in\mathbb{R}^{d\times d_h}\) such that
$M=A(B')^\top$ and $XMX^\top=XB X^\top$ for all admissible inputs \(X\).
Equivalently, for all token vectors \(u,v\in\mathbb{R}^d\), we have
$u^\top Mv=u^\top Bv$, which implies $M=B$.
However, since \(M=A(B')^\top\), we have
\begin{eqnarray}
\operatorname{rank}(M)\le d_h.
\end{eqnarray}
This contradicts
\begin{eqnarray}
\operatorname{rank}(B)>d_h.
\end{eqnarray}
Then, the score function \(X\mapsto XB X^\top\) belongs to
\(\mathcal{S}_{\mathrm{DAMCHA}}\) but not to
\(\mathcal{S}_{\mathrm{MHA}}\). Hence,
\begin{eqnarray}
\mathcal{S}_{\mathrm{MHA}}
\subsetneq
\mathcal{S}_{\mathrm{DAMCHA}}.
\end{eqnarray}
The proof is complete.
\end{proof}

% \paragraph{Interpretation.}
% Theorem~\ref{thm:strict_generalization} shows that standard MHA is a special
% case of DAMCHA obtained by imposing three restrictions: the metric is
% input-independent, the cross-head blocks are zero, and each head-wise metric is
% low-rank due to the query-key factorization. DAMCHA removes these restrictions
% by allowing input-adaptive full cross-head metrics, thereby inducing a strictly
% larger score-function class.

\subsection{Metric-Learning Advantage of Dynamic Cross-Head Metrics}
\label{subsec:theory_metric_learning}

The previous result establishes that DAMCHA strictly enlarges the attention
score-function class. We next explain why this enlargement is meaningful from
the perspective of metric learning.

The key difference between MHA and DAMCHA is that MHA learns a static global
metric, whereas DAMCHA learns an input-adaptive metric. In heterogeneous data,
the optimal similarity geometry may vary across contexts. A single global
metric must average over these different local geometries and therefore suffers
from an irreducible metric mismatch. DAMCHA can reduce this mismatch by
conditioning the metric on the input.

Let \(C\in\{1,\ldots,K\}\) denote a latent context variable with
\begin{eqnarray}
\mathbb{P}(C=c)=\pi_c>0,
\qquad
\sum_{c=1}^K \pi_c=1.
\end{eqnarray}
The context \(C\) may represent different semantic regimes, syntactic patterns,
visual structures, or distributional modes. We assume that, for each context
\(c\), there exists a context-specific optimal collection of effective
attention metrics
\begin{eqnarray}
\mathbf{M}_c^\star
=
\left(
M_{1,c}^\star,\ldots,M_{h,c}^\star
\right),
\end{eqnarray}
where \(M_{i,c}^\star\in\mathbb{R}^{d\times d}\) denotes the optimal effective
metric for the \(i\)-th head under context \(c\). Let $\ell(\mathbf{M};X,Y)$
denote a generic metric-learning loss associated with the effective metrics
\begin{eqnarray}
\mathbf{M}=(M_1,\ldots,M_h).
\end{eqnarray}
This loss may correspond to attention alignment loss, contrastive metric loss,
classification loss, or a downstream task loss induced by the attention scores.
For each context \(c\), define the conditional risk
\begin{eqnarray}
r_c(\mathbf{M})
=
\mathbb{E}
\left[
\ell(\mathbf{M};X,Y)\mid C=c
\right].
\end{eqnarray}

We impose the following assumptions.

\paragraph{Assumption 2.}
For each context \(c\), the conditional risk \(r_c(\mathbf{M})\) has a unique
minimizer \(\mathbf{M}_c^\star\), namely
\begin{eqnarray}
\mathbf{M}_c^\star
=
\arg\min_{\mathbf{M}} r_c(\mathbf{M}).
\end{eqnarray}

\paragraph{Assumption 3.}
There exists a constant \(\mu>0\) such that, for every context \(c\) and every
collection of metrics \(\mathbf{M}=(M_1,\ldots,M_h)\),
\begin{eqnarray}
r_c(\mathbf{M})-r_c(\mathbf{M}_c^\star)
\ge
\frac{\mu}{2}
\sum_{i=1}^h
\left\|M_i-M_{i,c}^\star\right\|_F^2.
\end{eqnarray}
This condition states that the conditional risk increases at least
quadratically as the learned metric deviates from the context-specific optimal
metric.

\paragraph{Assumption 4.}
The context \(C\) is identifiable from the input \(X\), namely there exists a
measurable function \(q(\cdot)\) such that
\begin{eqnarray}
C=q(X).
\end{eqnarray}
Moreover, the DAMCHA metric generator is expressive enough to represent the
context-dependent mapping
\begin{eqnarray}
X\mapsto \mathbf{M}_{q(X)}^\star.
\end{eqnarray}
Assumption~4 formalizes the idea that the input contains sufficient information
to determine the appropriate local similarity geometry.
For standard MHA, the effective metric of each head is static and belongs to
the low-rank query-key factorization class
\begin{eqnarray}
\mathcal{L}_{d_h}
=
\left\{
AB^\top:
A,B\in\mathbb{R}^{d\times d_h}
\right\}.
\end{eqnarray}
Thus, the optimal static MHA risk is
\begin{eqnarray}
R_{\mathrm{MHA}}^\star
=
\inf_{M_i\in\mathcal{L}_{d_h}}
\sum_{c=1}^K
\pi_c r_c(M_1,\ldots,M_h).
\end{eqnarray}
In contrast, DAMCHA can use input-adaptive effective metrics
\begin{eqnarray}
\mathbf{M}(X)
=
\left(
\bar M_1(X),\ldots,\bar M_h(X)
\right),
\end{eqnarray}
where each \(\bar M_i(X)\) is generated through the cross-head metric matrix.
The optimal DAMCHA risk is defined as
\begin{eqnarray}
R_{\mathrm{DAMCHA}}^\star
=
\inf_{g}
\mathbb{E}
\left[
\ell(g(X);X,Y)
\right],
\end{eqnarray}
where \(g(X)\) ranges over the metric mappings representable by DAMCHA.

\begin{theorem}[Metric-Learning Risk Advantage of DAMCHA]
\label{thm:metric_advantage}
Suppose Assumptions~2--4 hold. Then
\begin{eqnarray}
R_{\mathrm{DAMCHA}}^\star
\le
R_{\mathrm{MHA}}^\star.
\end{eqnarray}
Moreover,
\begin{eqnarray}
R_{\mathrm{MHA}}^\star
-
R_{\mathrm{DAMCHA}}^\star
\ge
\frac{\mu}{2}
\inf_{M_i\in\mathcal{L}_{d_h}}
\sum_{c=1}^K
\pi_c
\sum_{i=1}^h
\left\|
M_i-M_{i,c}^\star
\right\|_F^2.
\end{eqnarray}
Consequently, if there does not exist a single static collection of low-rank
metrics \((M_1,\ldots,M_h)\), with \(M_i\in\mathcal{L}_{d_h}\), such that
\begin{eqnarray}
M_i=M_{i,c}^\star
\qquad
\text{for all } i=1,\ldots,h
\text{ and all } c=1,\ldots,K,
\end{eqnarray}
then
\begin{eqnarray}
R_{\mathrm{DAMCHA}}^\star
<
R_{\mathrm{MHA}}^\star.
\end{eqnarray}
\end{theorem}

\begin{proof}
We first characterize the best risk attainable by an input-adaptive metric.
By Assumption~4, DAMCHA can represent the mapping $g^\star(X)=\mathbf{M}_{q(X)}^\star$.
Since \(C=q(X)\), this means that DAMCHA can select
$g^\star(X)=\mathbf{M}_{C}^\star$ for each input. Therefore, the risk achieved by this oracle adaptive metric is
\begin{eqnarray}
\mathbb{E}
\left[
\ell(g^\star(X);X,Y)
\right]
=
\sum_{c=1}^K
\pi_c
\mathbb{E}
\left[
\ell(\mathbf{M}_c^\star;X,Y)
\mid C=c
\right].
\end{eqnarray}
By the definition of \(r_c(\cdot)\), this becomes
\begin{eqnarray}
\mathbb{E}
\left[
\ell(g^\star(X);X,Y)
\right]
=
\sum_{c=1}^K
\pi_c r_c(\mathbf{M}_c^\star).
\end{eqnarray}

Since $R_{\mathrm{DAMCHA}}^\star$ is the infimum over all
DAMCHA-representable mappings, and $g^\star$ is one such mapping,
we have
\begin{eqnarray}
R_{\mathrm{DAMCHA}}^\star \;\leq\;
\mathbb{E}\bigl[\ell(g^\star(X);\,X,Y)\bigr]
= \sum_{c=1}^K \pi_c\, r_c(\mathbf{M}_c^\star).
\end{eqnarray}

Now consider standard MHA. It must use a static collection of metrics
\begin{eqnarray}
\mathbf{M}=(M_1,\ldots,M_h),
\qquad
M_i\in\mathcal{L}_{d_h},
\end{eqnarray}
independent of the input context. For any such \(\mathbf{M}\), its excess risk
over the adaptive oracle is
\begin{eqnarray}
\sum_{c=1}^K
\pi_c r_c(\mathbf{M})
-
\sum_{c=1}^K
\pi_c r_c(\mathbf{M}_c^\star)
=
\sum_{c=1}^K
\pi_c
\left[
r_c(\mathbf{M})-r_c(\mathbf{M}_c^\star)
\right].
\end{eqnarray}
By Assumption~3,
\begin{eqnarray}
r_c(\mathbf{M})-r_c(\mathbf{M}_c^\star)
\ge
\frac{\mu}{2}
\sum_{i=1}^h
\left\|
M_i-M_{i,c}^\star
\right\|_F^2.
\end{eqnarray}
Therefore,
\begin{eqnarray}
\sum_{c=1}^K
\pi_c r_c(\mathbf{M})
-
\sum_{c=1}^K
\pi_c r_c(\mathbf{M}_c^\star)
\ge
\frac{\mu}{2}
\sum_{c=1}^K
\pi_c
\sum_{i=1}^h
\left\|
M_i-M_{i,c}^\star
\right\|_F^2.
\end{eqnarray}
Taking the infimum over all static MHA metrics
\(M_i\in\mathcal{L}_{d_h}\), we obtain
\begin{eqnarray}
R_{\mathrm{MHA}}^\star
-
R_{\mathrm{DAMCHA}}^\star
\ge
\frac{\mu}{2}
\inf_{M_i\in\mathcal{L}_{d_h}}
\sum_{c=1}^K
\pi_c
\sum_{i=1}^h
\left\|
M_i-M_{i,c}^\star
\right\|_F^2.
\end{eqnarray}

The right-hand side is nonnegative, which proves
$R_{\mathrm{DAMCHA}}^\star
\le
R_{\mathrm{MHA}}^\star$. Furthermore, if no single static low-rank collection
\((M_1,\ldots,M_h)\) can simultaneously match all context-specific optimal
metrics \(\mathbf{M}_c^\star\), then
\begin{eqnarray}
\inf_{M_i\in\mathcal{L}_{d_h}}
\sum_{c=1}^K
\pi_c
\sum_{i=1}^h
\left\|
M_i-M_{i,c}^\star
\right\|_F^2
>
0.
\end{eqnarray}
Therefore, we have
$R_{\mathrm{MHA}}^\star-R_{\mathrm{DAMCHA}}^\star>0,$
which implies
$R_{\mathrm{DAMCHA}}^\star<
R_{\mathrm{MHA}}^\star.$
\end{proof}
Theorem~\ref{thm:metric_advantage} provides a metric-learning explanation for
the advantage of DAMCHA. Standard MHA learns a single static metric shared by
all inputs. When the data distribution contains multiple contexts with different
optimal similarity geometries, this static metric must compromise among them
and therefore suffers from an irreducible metric mismatch:
\begin{eqnarray}
\inf_{M_i\in\mathcal{L}_{d_h}}
\sum_{c=1}^K
\pi_c
\sum_{i=1}^h
\left\|
M_i-M_{i,c}^\star
\right\|_F^2.
\end{eqnarray}
In contrast, DAMCHA can condition its cross-head metric matrix on the input and
select a context-specific effective metric. Thus, DAMCHA behaves as an
input-adaptive local metric learner, whereas MHA behaves as a static global
metric learner.

\section{Computational Complexity Analysis of DAMCHA}
\label{app:complexity}
\subsection{Notation}

Throughout this section we use the following symbols consistently.

\begin{center}
\begin{tabular}{cl}
\hline
Symbol & Meaning \\
\hline
$L$   & Sequence (context) length \\
$d$   & Model dimension $d_{\mathrm{model}}$ \\
$h$   & Number of attention heads \\
$d_h$ & Per-head dimension, $d_h = d / h$ \\
$d_v$ & Per-head value dimension (set equal to $d_h$ unless noted) \\
$P$   & Depth (number of hidden layers) of the metric generator $f_\theta$ \\
$w$   & Hidden width of $f_\theta$ \\
$N$   & Number of stacked Transformer layers (depth of the network) \\
\hline
\end{tabular}
\end{center}

\noindent
All asymptotic statements use the standard $\mathcal{O}(\cdot)$ notation and suppress
constant factors.  We analyse one DAMCHA attention layer and one standard MHA
layer side-by-side, then aggregate to the full $N$-layer stack.

\subsection{Per-Layer Time Complexity}

A single DAMCHA attention layer consists of five distinct computational stages.

\paragraph{Stage 1: Metric generator $f_\theta(X)$.}
Given the input sequence $X \in \mathbb{R}^{L \times d}$, the metric generator first
computes a context summary $\bar{x} = \tfrac{1}{L}\sum_{\ell=1}^{L} X_\ell \in \mathbb{R}^{d}$
via mean pooling, at cost $\mathcal{O}(Ld)$.  The resulting vector is passed through
a $P$-layer MLP with hidden width $w$:
\begin{equation*}
  f_\theta:\; \mathbb{R}^{d}
  \;\xrightarrow{\;d \to w\;}\;
  \bigl(\mathbb{R}^{w}\bigr)^{P-2}
  \;\xrightarrow{\;w \to d^{2}\;}\;
  \mathbb{R}^{d^{2}},
\end{equation*}
producing the flattened cross-head metric matrix $M(X) \in \mathbb{R}^{d \times d}$.
The dominant costs are:
\begin{align*}
  \text{Input projection}  &:\quad \mathcal{O}(d \cdot w), \\
  (P-2) \text{ hidden layers} &:\quad \mathcal{O}((P-2)\,w^{2}), \\
  \text{Output projection} &:\quad \mathcal{O}(w \cdot d^{2}).
\end{align*}
Summing all stages of $f_\theta$ (including the pooling step):
\begin{equation}
  T_{f_\theta} = \mathcal{O}\!\left(Ld + Pw^{2} + wd^{2}\right).
  \label{eq:cost_fg}
\end{equation}
Crucially, under stack-wise parameter sharing $f_\theta$ is shared across all $N$
layers and need only be evaluated \emph{once} per forward pass, so its amortised
per-layer cost is $T_{f_\theta}/N$.

\paragraph{Stage 2: RowSum aggregation.}
The cross-head metric $M(X)$ is an $h \times h$ block matrix with blocks
$M_{i,j}(X) \in \mathbb{R}^{d_h \times d_h}$.  For each query head $i$ we compute
\begin{equation*}
  m_i(X) = \sum_{j=1}^{h} M_{i,j}(X) \;\in\; \mathbb{R}^{d_h \times d_h}.
\end{equation*}
This requires summing $h$ matrices of size $d_h \times d_h$ for each of $h$ heads:
\begin{equation}
  T_{\mathrm{RowSum}}
  = \mathcal{O}\!\left(h^{2} d_h^{2}\right)
  = \mathcal{O}\!\left(d^{2}\right),
\end{equation}
since $h \cdot d_h = d$.

\paragraph{Stage 3: Attention score computation.}
For each head $i$, let $X_i = X[\,:\,,\, i d_h : (i{+}1)d_h] \in \mathbb{R}^{L \times d_h}$
be the $i$-th positional slice of the input.  The pre-softmax score matrix is
\begin{equation*}
  S_i(X) = \frac{1}{\sqrt{d_k}}\, X_i\, m_i(X)\, X_i^{\top} \;\in\; \mathbb{R}^{L \times L}.
\end{equation*}
The two matrix products cost:
\begin{align*}
  X_i m_i(X) &:\quad \mathcal{O}(L d_h^{2}), \\
  \bigl(X_i m_i(X)\bigr) X_i^{\top} &:\quad \mathcal{O}(L^{2} d_h).
\end{align*}
Summing over all $h$ heads:
\begin{equation}
  T_{\mathrm{score}}
  = \mathcal{O}\!\left(h \bigl(L d_h^{2} + L^{2} d_h\bigr)\right)
  = \mathcal{O}\!\left(\frac{Ld^{2}}{h} + L^{2}d\right).
\end{equation}

\paragraph{Stage 4: Softmax.}
Row-wise softmax over $h$ score matrices of size $L \times L$:
\begin{equation}
  T_{\mathrm{softmax}} = \mathcal{O}(hL^{2}).
\end{equation}

\paragraph{Stage 5: Value projection and output.}
This stage is structurally identical to standard MHA.  Value projections
$V_i = X W_V^{(i)}$ cost $\mathcal{O}(L d d_v)$ per head; the attention-weighted
sum $A_i V_i$ costs $\mathcal{O}(L^{2} d_v)$ per head; and the output projection
$W_O$ costs $\mathcal{O}(L d^{2})$.  With $d_v = d_h$ and summing over heads:
\begin{equation}
  T_{\mathrm{value}} = \mathcal{O}(L d^{2} + L^{2} d).
\end{equation}

\paragraph{Total per-layer cost of DAMCHA.}
Combining the five stages, and writing $T_{f_\theta}^{\mathrm{amort}} =
T_{f_\theta}/N$ for the shared generator:
\begin{equation}
  T_{\mathrm{DAMCHA}}^{\mathrm{layer}}
  = \mathcal{O}\!\left(\frac{Ld^{2}}{h} + L^{2}d\right)
  + \frac{1}{N}\,\mathcal{O}\!\left(Ld + Pw^{2} + wd^{2}\right).
  \label{eq:damcha_total}
\end{equation}

\subsection{Comparison with Standard MHA}

Standard MHA per layer involves:
\begin{align*}
  \text{Q, K projections (per head, }h\text{ heads)} &:\quad \mathcal{O}(Ld \cdot d_h \cdot 2h) = \mathcal{O}(Ld^{2}), \\
  \text{V projection} &:\quad \mathcal{O}(Ld^{2}), \\
  QK^{\top}\text{ (per head)} &:\quad \mathcal{O}(L^{2}d_h \cdot h) = \mathcal{O}(L^{2}d), \\
  \text{Softmax} &:\quad \mathcal{O}(hL^{2}), \\
  \text{Output projection} &:\quad \mathcal{O}(Ld^{2}).
\end{align*}
\begin{equation}
  T_{\mathrm{MHA}}^{\mathrm{layer}} = \mathcal{O}(Ld^{2} + L^{2}d).
\end{equation}

\begin{table}[h]
\centering
\caption{Per-layer time complexity: DAMCHA vs.\ standard MHA. $\dagger$ indicates
the stage is shared across all $N$ layers and amortises to $\mathcal{O}(\cdot)/N$.}
\label{tab:complexity_compare}
\begin{tabular}{lccc}
\hline
Stage & MHA & DAMCHA \\
\hline
Q, K projections   & $\mathcal{O}(Ld^{2})$ & --- (replaced by $f_\theta$) \\
Metric generator $f_\theta$$\dagger$ & --- & $\mathcal{O}(Ld + Pw^{2} + wd^{2})$ \\
RowSum aggregation & --- & $\mathcal{O}(d^{2})$ \\
Attention scores   & $\mathcal{O}(Ld^{2} + L^{2}d)$ & $\mathcal{O}(Ld^{2}/h + L^{2}d)$ \\
Softmax            & $\mathcal{O}(hL^{2})$ & $\mathcal{O}(hL^{2})$ \\
V projection + output & $\mathcal{O}(Ld^{2} + L^{2}d)$ & $\mathcal{O}(Ld^{2} + L^{2}d)$ \\
\hline
\textbf{Total (per layer)} & $\mathcal{O}(Ld^{2} + L^{2}d)$ &
  $\mathcal{O}(Ld^{2}/h + L^{2}d)$ + amort.\ $f_\theta$ \\
\hline
\end{tabular}
\end{table}

Table~\ref{tab:complexity_compare} summarises the comparison.
Three observations follow directly from this comparison.

\textbf{Score computation is cheaper.}  The dominant cost of attention score
formation in MHA is $\mathcal{O}(Ld^{2})$ (from Q, K projections plus
$QK^{\top}$).  In DAMCHA, the corresponding cost is $\mathcal{O}(Ld^{2}/h)$,
a factor-of-$h$ saving, because $f_\theta$ directly outputs the effective metric
$m_i(X)$ without the $d \to d_h$ projection step.

\textbf{Overhead of $f_\theta$ is amortised.}  Under stack-wise parameter sharing,
$f_\theta$ is invoked once per forward pass regardless of $N$.  As $N$ grows
(the typical setting for deep Transformers), its per-layer contribution tends to
zero.

\subsection{Parameter Count}

Let $\Theta(\cdot)$ denote the exact parameter count.

\paragraph{Standard MHA.}  Each layer contains projections
$W_Q^{(i)}, W_K^{(i)}, W_V^{(i)} \in \mathbb{R}^{d \times d_h}$ for all $h$ heads
and an output projection $W_O \in \mathbb{R}^{d \times d}$:
\begin{equation}
  \Theta_{\mathrm{MHA}}^{\text{layer}} = 3h \cdot d \cdot d_h + d^{2}
  = 3d^{2} + d^{2} = 4d^{2}.
\end{equation}
For an $N$-layer stack: $\Theta_{\mathrm{MHA}} = 4Nd^{2}$.

\paragraph{DAMCHA.}  Each layer retains $W_V^{(i)}$ and $W_O$ but discards
$W_Q^{(i)}$ and $W_K^{(i)}$, which are superseded by $f_\theta$:
\begin{equation}
  \Theta_{\mathrm{DAMCHA}}^{\text{layer}} = h \cdot d \cdot d_h + d^{2} = 2d^{2}.
\end{equation}
The shared metric generator contributes a \emph{single} parameter block:
\begin{equation}
  \Theta_{f_\theta} = d \cdot w + (P-2)w^{2} + w \cdot d^{2}
  = \mathcal{O}(Pw^{2} + wd^{2}).
\end{equation}
Total for the $N$-layer stack:
\begin{equation}
  \Theta_{\mathrm{DAMCHA}} = 2Nd^{2} + \Theta_{f_\theta}.
\end{equation}

\paragraph{Net saving.}
\begin{equation}
  \Delta\Theta = \Theta_{\mathrm{MHA}} - \Theta_{\mathrm{DAMCHA}}
  = 2Nd^{2} - \Theta_{f_\theta}.
\end{equation}
This saving is strictly positive whenever
$2Nd^{2} > \Theta_{f_\theta}$, i.e.\
$N > \Theta_{f_\theta} / (2d^{2}) \approx Pw^{2}/(2d^{2}) + w/2$.
For the GPT-2 scale configuration ($N=12$, $d=768$, $w=1536$, $P=9$) the
right-hand side evaluates to approximately $12.8$, so DAMCHA has essentially the
same parameter count as MHA at this scale and fewer parameters for larger $N$.
Concretely, DAMCHA achieves a $\sim 25\%$ reduction in attention-module parameters
relative to MHA when $N \geq 15$ and $w \approx 2d$.

% \subsection{Full-Stack Complexity Summary}

% Table~\ref{tab:fullstack} consolidates all results for an $N$-layer Transformer.

% \begin{table}[h]
% \centering
% \caption{Full $N$-layer stack complexity.  $f_\theta$ cost in DAMCHA is a
% one-time cost, not replicated per layer.}
% \label{tab:fullstack}
% \begin{tabular}{lcc}
% \hline
% Quantity & MHA ($N$ layers) & DAMCHA ($N$ layers) \\
% \hline
% Time (forward pass)
%   & $\mathcal{O}(N(Ld^{2} + L^{2}d))$
%   & $\mathcal{O}(N(Ld^{2}/h + L^{2}d)) + \mathcal{O}(Ld + Pw^{2} + wd^{2})$ \\
% Parameters
%   & $4Nd^{2}$
%   & $2Nd^{2} + \mathcal{O}(Pw^{2} + wd^{2})$ \\
% Activation memory
%   & $\mathcal{O}(N(hL^{2} + Ld))$
%   & $\mathcal{O}(N(hL^{2} + Ld)) + \mathcal{O}(d^{2})$ \\
% \hline
% \end{tabular}
% \end{table}

% \noindent
% The asymptotic complexity of DAMCHA is dominated by the same $\mathcal{O}(L^{2}d)$
% term as standard MHA (arising from the attention-weight computation that is
% unavoidable in any $L$-length self-attention mechanism), while the DAMCHA-specific
% score-formation cost $\mathcal{O}(Ld^{2}/h)$ is strictly lower than MHA's
% $\mathcal{O}(Ld^{2})$ counterpart.  Combined with the one-time amortised cost of
% $f_\theta$ and the flat $\mathcal{O}(d^{2})$ memory overhead, DAMCHA attains
% strictly greater representational capacity than MHA with comparable or reduced
% computational overhead across all three axes of complexity.

\section{Additional Distribution Plots for Sample Generation}
\label{app:distplots}

Figures~\ref{fig:app_dist_cifar10}--\ref{fig:app_dist_commongen} show that DAMCHA
consistently changes the \emph{shape} of the generation-score distributions rather
than only improving the average. Across datasets, DAMCHA tends to shift probability
mass toward the high-score region and reduce the density of low-score samples,
indicating fewer low-quality outliers. In addition, the KDE curves under DAMCHA are
typically more concentrated, suggesting improved stability (smaller dispersion) in
generation quality. These distribution-level patterns align with the main-text
quantitative results and provide evidence that the gains are not driven by a small
subset of samples.

\begin{figure}[h]
    \centering
    \includegraphics[width=0.75\textwidth]{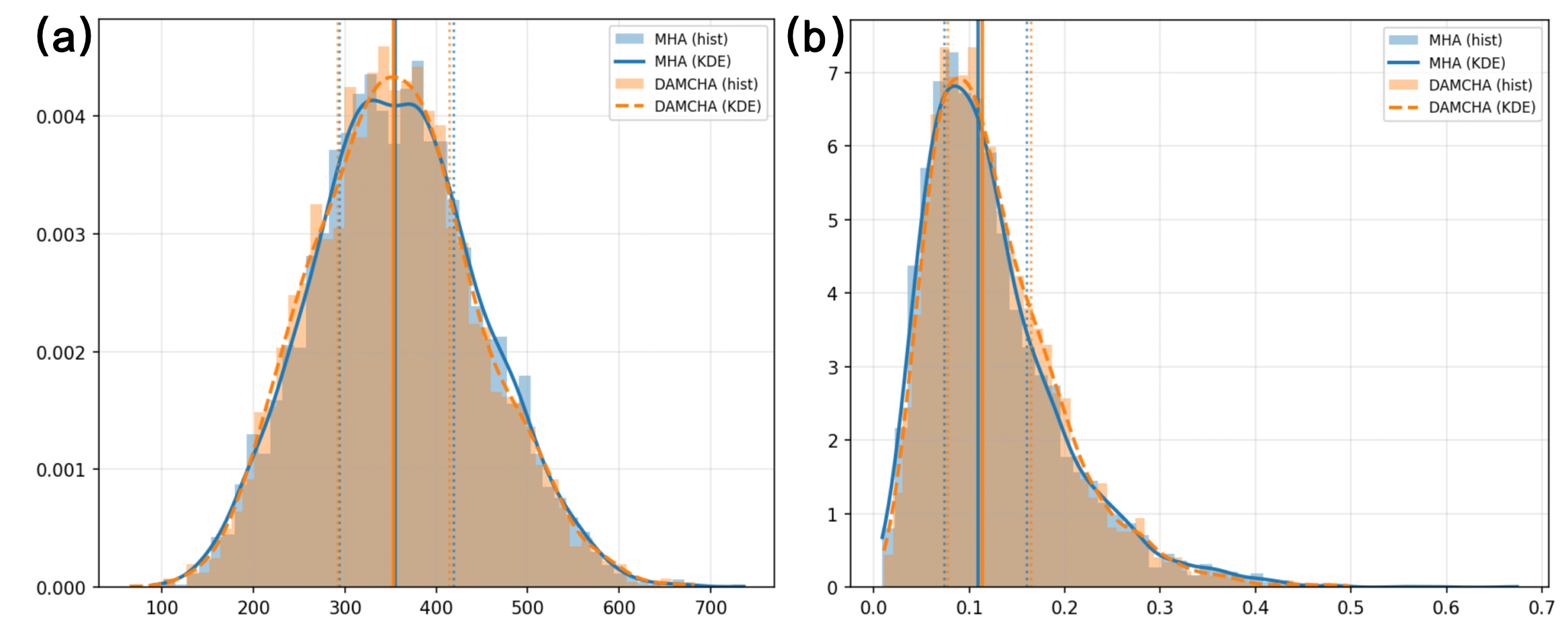}
    \caption{Supplementary distribution plot on CIFAR-100.}
    \label{fig:app_dist_cifar100}
\end{figure}

\begin{figure}[h]
    \centering
    \includegraphics[width=0.75\textwidth]{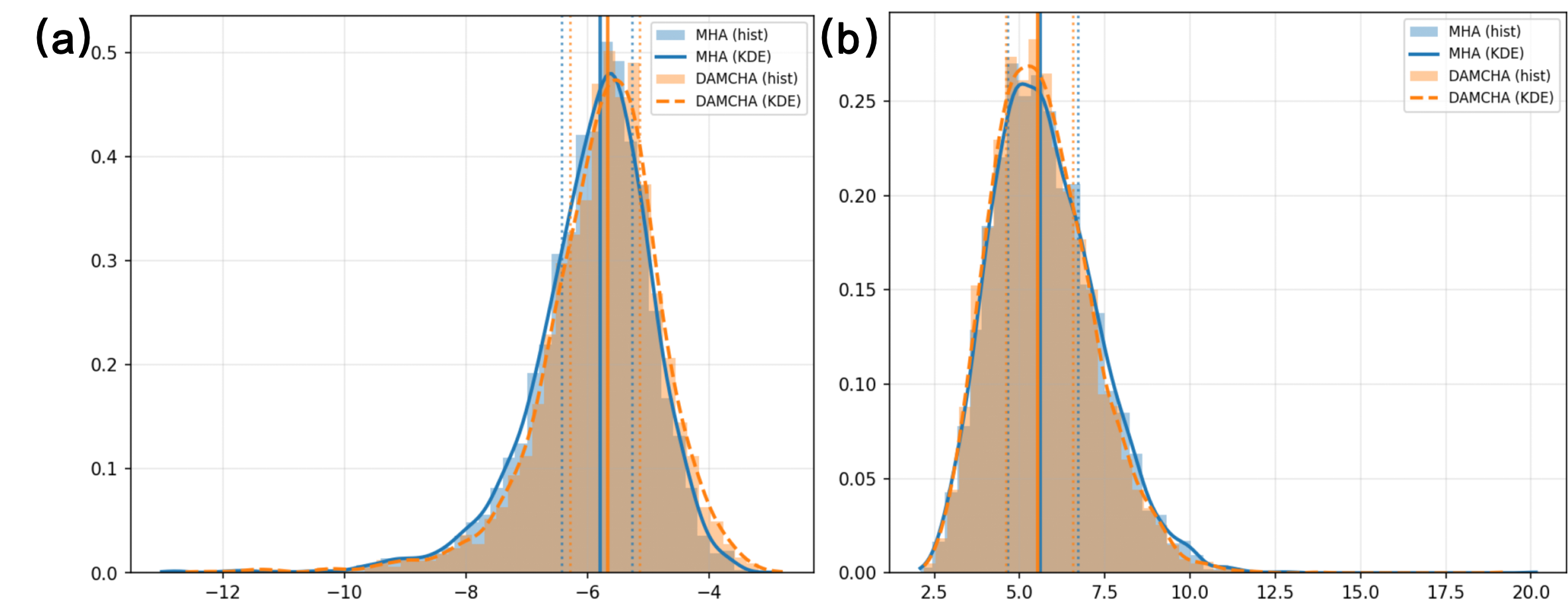}
    \caption{Supplementary distribution plot on CommonGen.}
    \label{fig:app_dist_commongen}
\end{figure}

\newpage

\section{The Evolution of Cross-Head Metric Matrix during training}
\label{app:Appendix F}

Figure~\ref{fig:cifar10_m_matrix} visualizes the learned cross-head metric matrices $M(X)$ at different training stages on CIFAR-10. We show four representative layers (layer-0, 1, 2, 3) at epoch 1 (top row) and epoch 50 (bottom row). At epoch 1, the matrices exhibit relatively uniform structures with weak off-diagonal interactions. By epoch 50, three distinct patterns emerge: 
\textbf{(1) Strong diagonal blocks:} The diagonal $8\times8$ blocks corresponding to $M_{i,i}(X)$ show high activation values, indicating well-learned intra-head similarities that capture head-specific feature patterns. 
\textbf{(2) Structured off-diagonal blocks:} Systematic cross-head interactions appear in the off-diagonal regions, particularly in middle layers (layer-1, layer-2). These blocks encode how queries from one head measure similarity under the feature patterns of another head, implementing the cross-head cooperation mechanism. 
\textbf{(3) Layer-specific specialization:} Early layers (layer-0, layer-1) develop denser and more diverse off-diagonal structures, suggesting active cross-modal feature integration in shallow representations. Deeper layers (layer-2, layer-3) maintain sparser but more targeted interactions, focusing on task-specific head coordination. This hierarchical specialization aligns with the intuition that early layers perform general feature extraction while deeper layers implement task-specific reasoning.

\begin{figure*}[h]
  \centering
  \includegraphics[width=\textwidth]{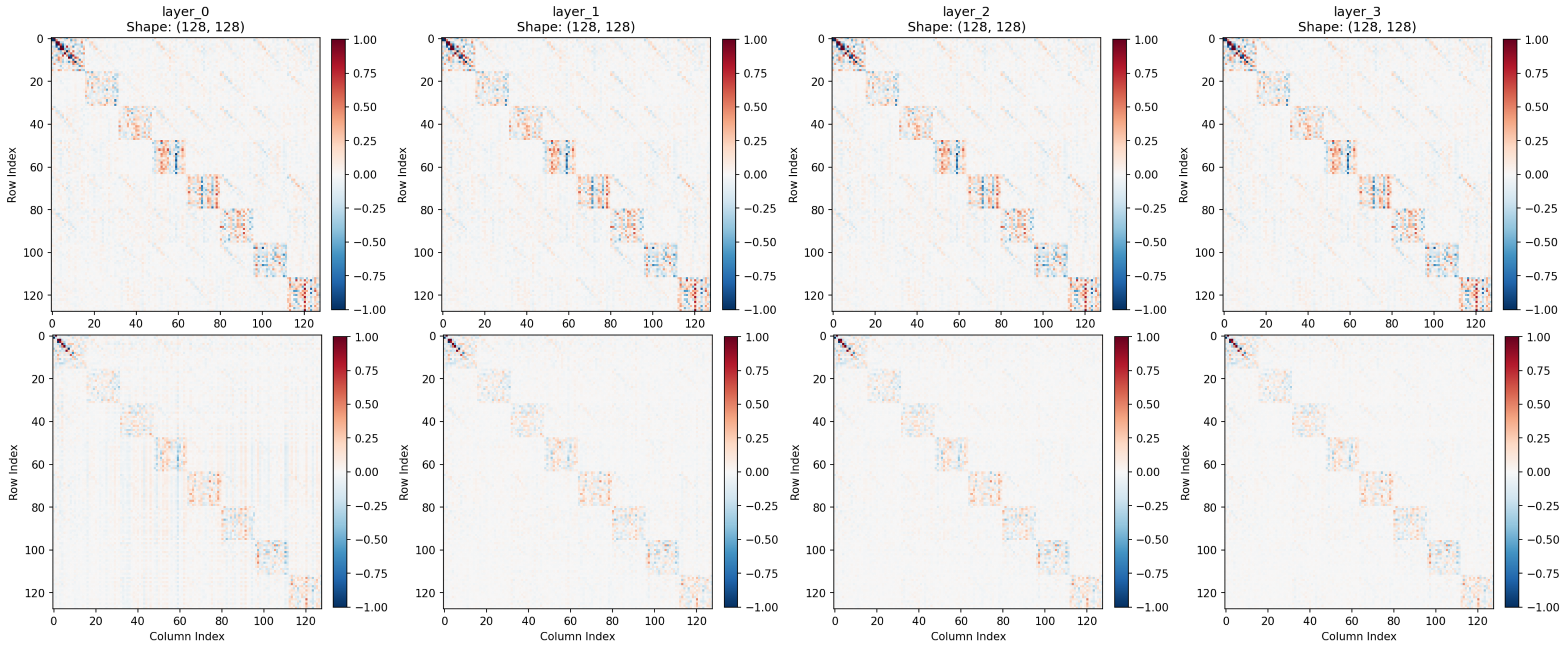}
  \caption{Cross-head metric matrix visualization on CIFAR-10. The top row shows the state of the cross-head metric matrix at epoch 1, and the bottom row shows the state at epoch 50.}
  \label{fig:cifar10_m_matrix}
\end{figure*}

% =========================
% Appendix B
% =========================

\section{Gradient Flow Analysis for Stack-Wise Sharing}
\label{app:Appendix G}

To understand how parameter sharing stabilizes training, Figure~\ref{fig:ablation_grad} visualizes the gradient contribution heatmaps of the cross-head metric matrix under two settings: (a) without sharing and (b) with the standard DAMCHA sharing mechanism. Each heatmap shows the $L_2$ norm of gradients flowing into each block of the metric matrix $M \in \mathbb{R}^{hd \times hd}$ during training.

Without sharing (panel a), gradients are distributed unevenly across the matrix, with some blocks receiving disproportionately large updates (bright regions) while others remain nearly inactive (dark regions). This uneven flow leads to unbalanced learning where certain cross-head interactions dominate while others fail to develop. In contrast, with parameter sharing (panel b), gradients distribute more uniformly across the entire matrix. The diagonal blocks maintain strong but controlled gradients, while off-diagonal blocks show moderate and consistent activation. This balanced gradient flow prevents any single layer from learning degenerate or overly specialized metrics, promoting cooperative learning across layers.

The visualization explains the performance gap in Table~\ref{tab:ablation}: shared parameters act as an implicit regularizer, forcing all layers to reuse a consistent family of metric transformations. This constraint prevents layer-specific overfitting and encourages the model to learn generalizable cross-head interaction patterns. Qualitative contour-map comparisons across epochs are included in Appendix~\ref{app:ablationviz}
(Fig.~\ref{fig:app_contour_ablation}), supporting the ablation findings.

\begin{figure}[h]
  \centering
  \includegraphics[width=0.8\textwidth]{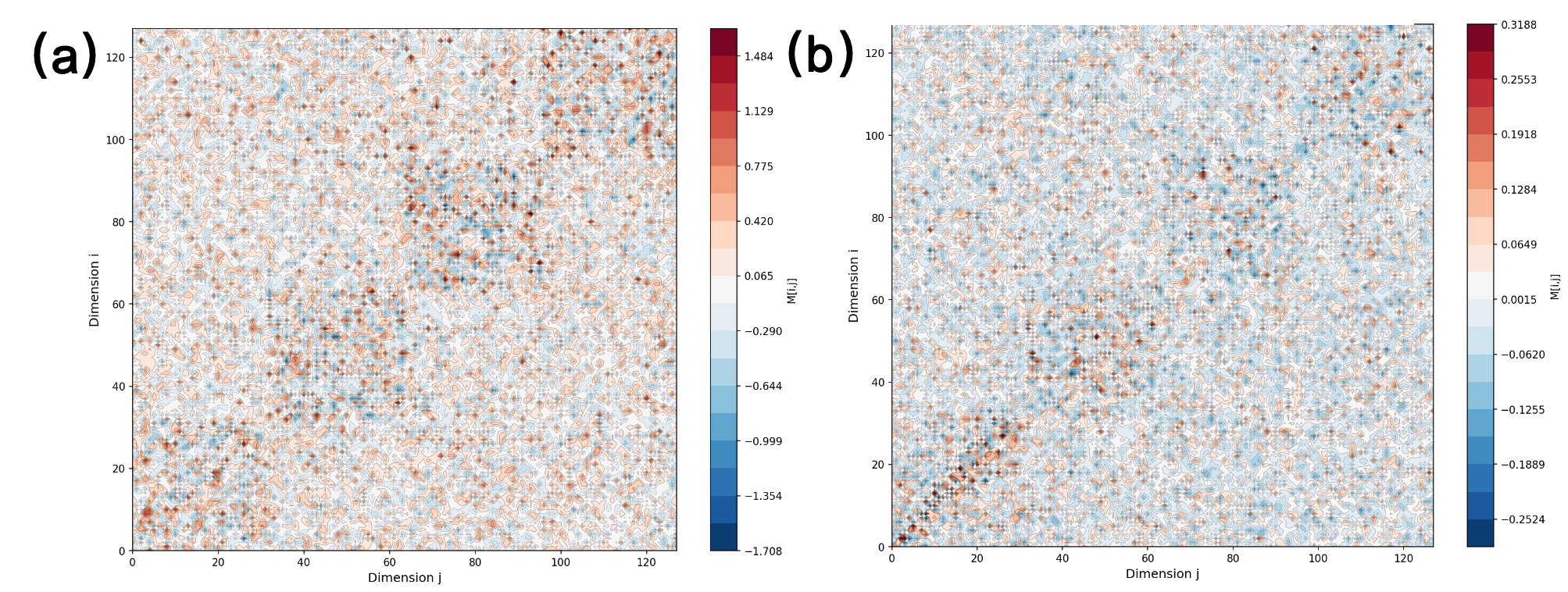}
  \caption{(a) Gradient contribution heatmap of the cross-head metric matrix when the sharing mechanism is disabled. (b) Gradient contribution heatmap of the cross-head metric matrix under the standard DAMCHA setting.}
  \label{fig:ablation_grad}
\end{figure}

\section{Additional Visualizations for Ablation and Metric Geometry}
\label{app:ablationviz}

Figure~\ref{fig:app_contour_ablation} provides a qualitative ablation comparison
across training epochs for WMT dataset. The w/o $f(x)$ variant tends to exhibit less coherent contour
structures and slower stabilization over epochs, whereas DAMCHA forms clearer and
more stable geometric patterns as training progresses. This supports the view that
the input-adaptive generator $f(x)$ helps shape a more structured and stable metric
geometry.

\begin{figure}[h]
    \centering
    \includegraphics[width=\textwidth]{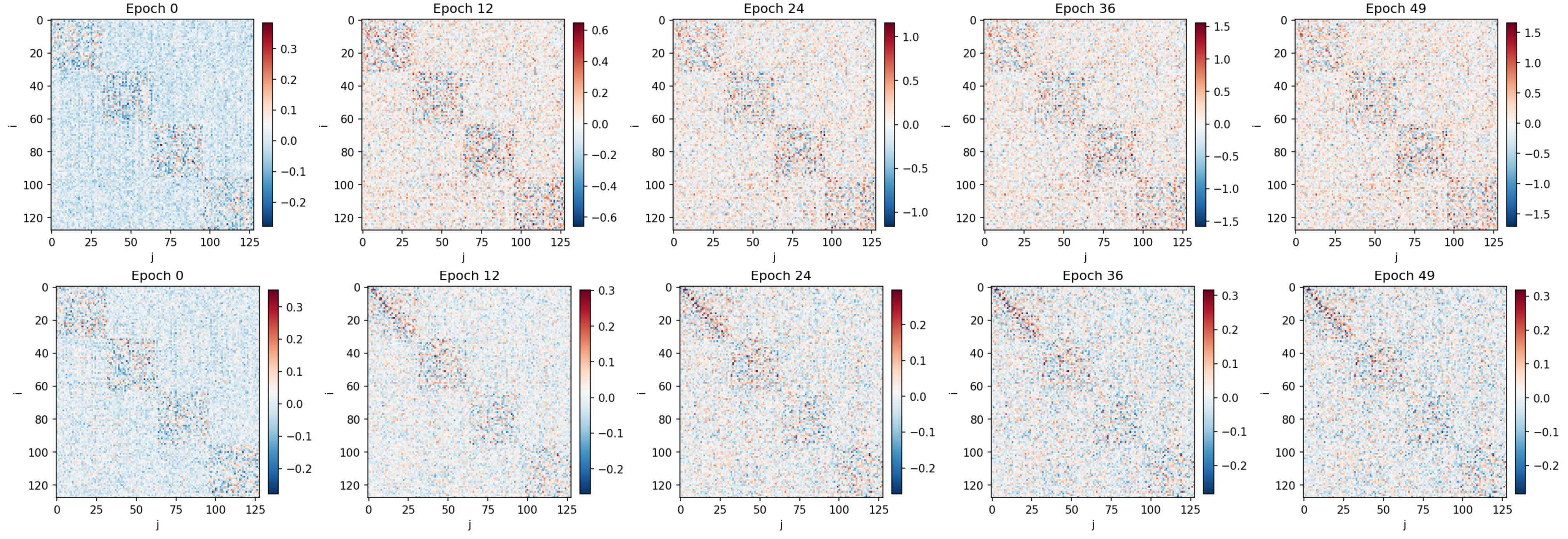}
    \caption{Contour-map visualizations across training epochs (Epoch 0/12/24/36/49). Top: w/o $f(x)$; Bottom: DAMCHA.}
    \label{fig:app_contour_ablation}
\end{figure}

Figure~\ref{fig:app_M_3d} further corroborates this observation from the perspective
of the learned metric matrix $M$. Compared with the static design (w/o $f(x)$), the
DAMCHA surface typically appears more organized (with smoother variations and fewer
irregular spikes), consistent with a better-conditioned metric landscape.

\begin{figure}[h]
    \centering
    \includegraphics[width=0.8\textwidth]{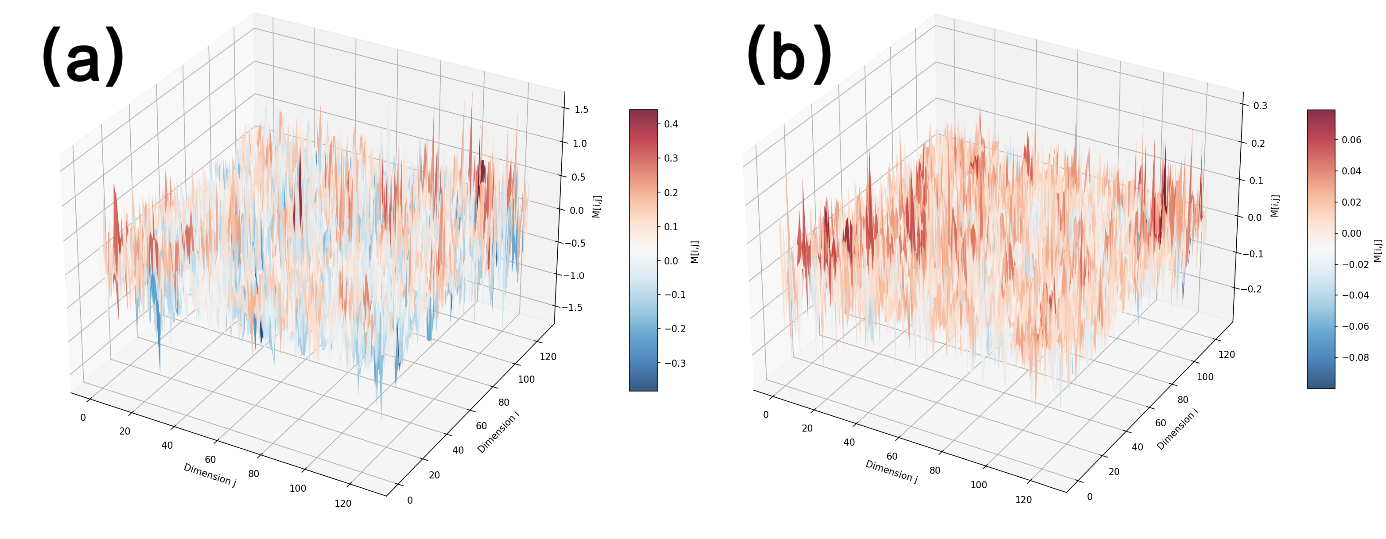}
    \caption{3D surface visualization of the learned metric matrix $M$. Left: w/o $f(x)$; Right: DAMCHA.}
    \label{fig:app_M_3d}
\end{figure}

\section{Synthetic Tasks}
\label{app:synthetic}
In this section, we design two synthetic tasks (regression and classification) to more explicitly present the advantages and underlying mechanism of our cross-head metric learning framework: modeling dependencies that require coordinated attention across multiple subspaces.
\par
We consider a feature space formed by 4 subspaces. Each subspace is constructed by a vector-valued variable, respectively, denoted by $\{\boldsymbol{x}_i\}_{i=1}^{4}$. Each $\boldsymbol{x}_i\in \mathbb{R}^{d_e}$ and its data is generated following $\boldsymbol{x}_i \stackrel{\text{i.i.d.}}{\sim} \mathcal{N}(\boldsymbol{0}, \boldsymbol{I}).$ We define the following 4-dimensional regression system
\begin{equation}
y_i = w_i^{\top} \boldsymbol{x}_i + \lambda \sum_{j \neq i} (\boldsymbol{x}_i^{\top} W_{ij} \boldsymbol{x}_j),
\end{equation}
where $w_i \in \mathbb{R}^{d_{e}}$ is a within-subspace weight vector, $W_{ij} \in \mathbb{R}^{d_{e} \times d_{e}}$ is a cross-subspace interaction matrix. All the weight parameters are initialized by sampling from $\mathcal{N}(0,0.25)$. When $\lambda = 0$, the task reduces to independent within-subspace prediction. 
As $\lambda$ increases, cross-subspace interactions become more critical, and DAMCHA's cross-head metric can provide advantages.

\par
For a $C$-classes problem, we define a target classification system. The score for class $c$ is computed as:
\begin{equation}
\text{score}_c = \sum_{i < j} (x_i^{\top} W_{ij}^c x_j),
\end{equation}
where $W_{ij}^c$ is a class-specific cross-subspace interaction matrix. $W_{ij}^c$ is constructed as: $W_{ij}^c = \text{scale} \cdot U \odot V$ where $U, V \in \mathbb{R}^{d_{e} \times d_{e}}$ are randomly initialized and $\text{Rank}(W_{ij}^c) = d_{e} / 2$. The final label is determined by $\text{label} = \arg\max_c (\text{score}_c + \text{noise})$ where noise is added for regularization.
$\lambda$ controls the entanglement level (strength of cross-subspace dependencies). 
\par
\begin{figure}[t]
  \centering
  \includegraphics[width=0.8\textwidth]{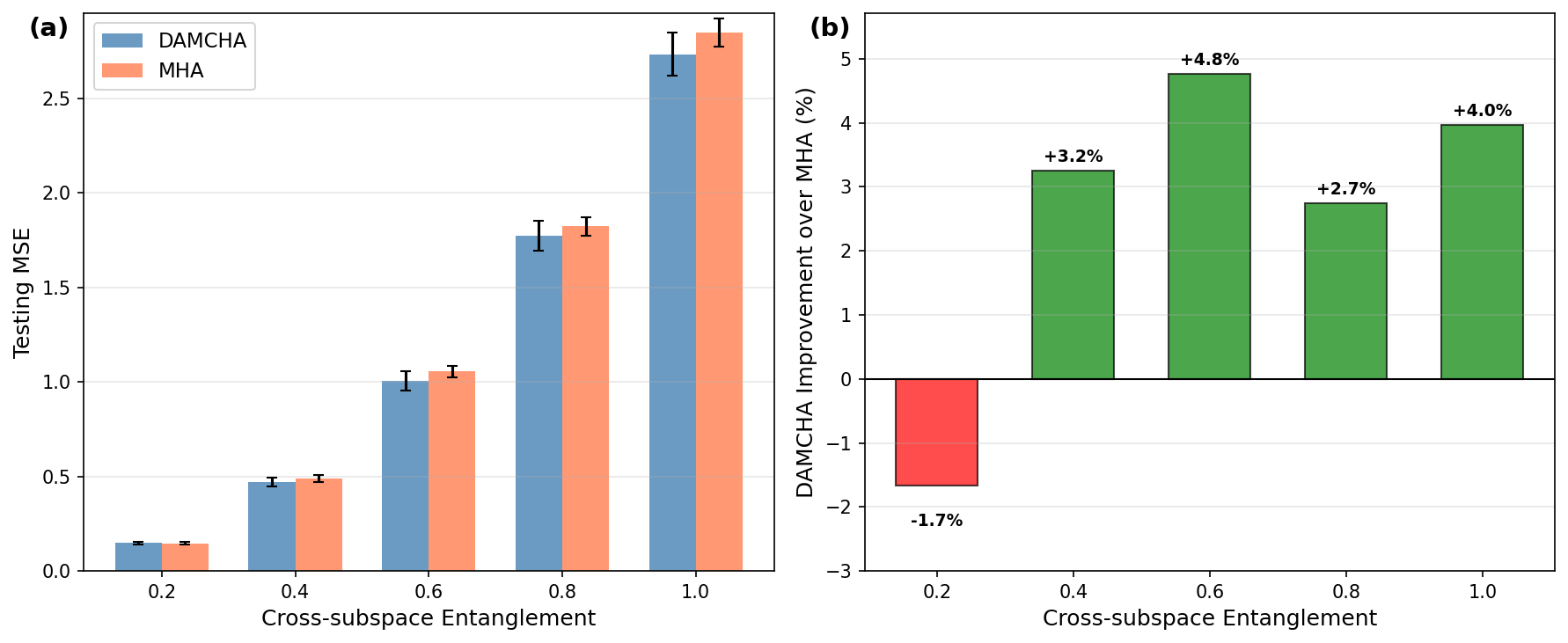}
  \caption{Cross-subspace regression: DAMCHA vs. MHA.
  (a) comparison results (testing MSE) under different entanglement levels.
  (b) DAMCHA's relative advantage under different entanglement levels. }
  \label{fig:cross_group_interaction}
  \vspace{-1em}
\end{figure}
For both systems, we train two 4-head Transformer models with head dimension 32. For the regression system, under $\lambda \in \{0.2, 0.4, 0.6, 0.8, 1.0\}$, the model is trained for 10K steps with batch size 64. The model for the classification system is trained with 14K steps and batch size 128.
\par
For regression task, we evaluate DAMCHA and MHA using testing mean square error (MSE). Figure \ref{fig:cross_group_interaction}(a) presents comparison results with error bars, which demonstrates clear advantages of DAMCHA.
Figure \ref{fig:cross_group_interaction}(b) reveals a more nuanced picture by plotting DAMCHA's relative advantage over MHA. At $\lambda = 0.2$, DAMCHA underperforms MHA by 1.7\%, when only limited cross-subspace interaction is needed, As $\lambda$ increases from 0.2 to 1, DAMCHA's advantage grows from $-1.7\%$ to $+4.8\%$, confirming that the benefit scales with the strength of subspace entanglement.

\par
\begin{figure}[t]
  \centering
  \includegraphics[width=0.8\textwidth]{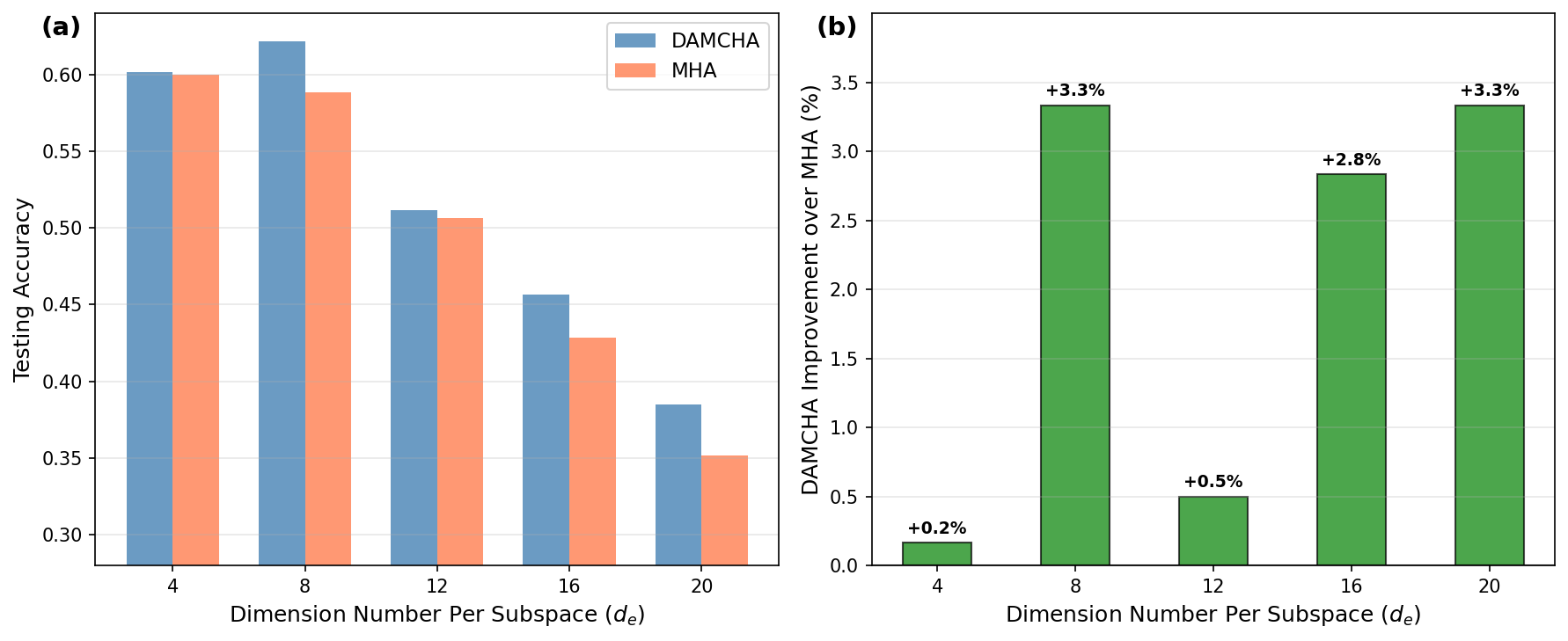}
  \caption{Cross-subspace classification: DAMCHA vs. MHA.
  (a) Comparison results (accuracy \%) under different $d_e$.
  (b) Relative improvement of DAMCHA over MHA across different $d_e$.}
  \label{fig:cross_group_classification}
  \vspace{-1em}
\end{figure}
For classification task, for comprehensive and robust evaluation, we vary $d_{e} \in \{4, 8, 12, 16, 20\}$,  resulting in total dimensions from 16 to 80. Figure \ref{fig:cross_group_classification}(a) provides accuracy comparison against $d_e$, visually confirming DAMCHA's consistent advantage across all settings. It also reveals an important scaling property: as $d_e$ increases, both methods degrade, but DAMCHA degrades more slowly. This implies that DAMCHA's cross-head coordination becomes increasingly valuable in high-dimensional settings where MHA's independent heads cannot effectively capture the full feature space interactions. Figure \ref{fig:cross_group_classification}(b) quantifies the relative improvement, further illustrating the advantages of DAMCHA.

% The presented two synthetic experiments establish clear conditions under which DAMCHA provides advantage: (i) when target dependencies involve interactions across feature subspaces assigned to different heads, (ii) when such cross-subspace interactions are sufficiently strong ($\lambda \geq 0.2$ in the regression system), and (iii) when feature dimensionality is moderate to high ($d_{e} \geq 8$ in the classification system). These findings explain DAMCHA's strong performance on real-world tasks (Sections 4.1-4.3) where complex, multi-faceted dependencies naturally arise. They also explain scenarios where DAMCHA's advantage is modest: when data is simple or when cross-head coordination is unnecessary, by simpler inductive bias, MHA may provide comparable generalization results.
The two synthetic experiments identify clear conditions under which DAMCHA provides advantage: (i) when the target dependencies explicitly involve interactions across feature subspaces that are assigned to different attention 
heads, (ii) such cross-subspace interactions are sufficiently strong (e.g., $\lambda \geq 0.2$ in the regression system), and (iii) the feature dimensionality is moderate to high (e.g., $d_e \geq 8$ in the classification system), where richer representational capacity becomes critical.
These controlled findings offer a principled explanation for DAMCHA's strong empirical performance on real-world benchmarks, where complex, multi-faceted dependencies naturally arise. In such settings, the ability to explicitly model cross-head interactions allows DAMCHA to capture dependencies.
When the underlying data-generating process is relatively simple, or when meaningful dependencies can be captured within individual heads without requiring cross-head coordination, standard MHA may already be sufficient, leading to comparable generalization performance. This observation highlights that DAMCHA is most beneficial in scenarios where cross-subspace interactions are intrinsic to the task, while remaining competitive in simpler regimes.

\section{Computational Resources and Runtime Details}
\label{app:compute}

All experiments in this paper were conducted on a single server with the following
hardware and software configuration:

\begin{itemize}
    \item \textbf{GPU:} NVIDIA RTX 3090 (24\,GB VRAM) $\times$ 1
    \item \textbf{CPU:} 12\,vCPU Intel(R) Xeon(R) Platinum 8255C @ 2.50\,GHz
    \item \textbf{Operating System:} Ubuntu 22.04
    \item \textbf{Python:} 3.12
    \item \textbf{PyTorch:} 2.7.0
    \item \textbf{CUDA:} 12.8
\end{itemize}

\textbf{Per-Experiment Runtime.}
Table~\ref{tab:runtime} summarizes the approximate wall-clock training time for each
experimental setting.

\begin{table}[h]
    \centering
    \small
    \caption{Approximate training time per experimental setting. All runs use a single
    RTX~3090 GPU.}
    \label{tab:runtime}
    \begin{tabular}{llrr}
        \toprule
        Experiment & Dataset & Steps / Epochs & Approx.\ Training Time \\
        \midrule
        \multirow{4}{*}{Decoder-only Transformer}
            & CIFAR-10    & 50 epochs  & $\sim$3\,h \\
            & CIFAR-100   & 50 epochs  & $\sim$3\,h \\
            & WMT14       & 50 epochs  & $\sim$5\,h \\
            & CommonGen   & 50 epochs  & $\sim$4\,h \\
        \midrule
        \multirow{3}{*}{GPT-2 Pretraining}
            & Enwik8      & 100K steps & $\sim$8\,h \\
            & Text8       & 100K steps & $\sim$8\,h \\
            & Mix         & 100K steps & $\sim$10\,h \\
        \midrule
        ViT-S/16 Classification & ImageNet-1K & 500 epochs & $\sim$72\,h \\
        \midrule
        Synthetic Tasks & Regression / Classification & 10K--14K steps & $<$30\,min \\
        \bottomrule
    \end{tabular}
\end{table}

\section{Comparison with Baseline Methods}
\label{app:comparison}
\paragraph{Mechanistic Advantages over Competing Attention Methods.}
From a theoretical standpoint, the advantages of DAMCHA over all baseline methods can be
understood through three independent but complementary mechanisms:
the \emph{structural expansion} of the metric feasible class,
the \emph{dynamic adaptivity} of the learned similarity geometry, and
the \emph{rank enhancement} induced by cross-head aggregation.
No existing baseline addresses all three simultaneously, and the gap between DAMCHA and each
competitor can be precisely traced back to which of these mechanisms it lacks.
At the empirical level, these distinctions consistently translate into superior performance
across all benchmark tasks, stemming from two fundamental innovations absent in every prior method:
\emph{(i)} learning the full cross-head metric matrix~$\mathbf{M}(X)$ that explicitly captures
inter-head interactions beyond the block-diagonal constraint of standard MHA, and
\emph{(ii)} parameterizing this metric as an input-dependent function, enabling the attention
geometry to dynamically adapt to the local manifold structure of the data at inference time.

\textbf{vs.\ MHA.}
Standard MHA imposes a block-diagonal restriction on the joint head-space metric,
structurally zeroing out all off-diagonal interactions, fixing the metric after training,
and capping the per-head rank at the projection dimension,
thereby simultaneously introducing all three deficiencies that DAMCHA is designed to resolve.
DAMCHA strictly generalizes MHA: MHA is exactly recoverable as a special case by setting
all off-diagonal metric blocks to zero and using a constant generator, while the converse
is provably false ($\mathcal{F}_{\text{MHA}} \subsetneq \mathcal{F}_{\text{DAMCHA}}$).
The resulting performance gap decomposes into two independent additive components:
a structural gap from the enlarged metric feasible class, and a dynamic gap from
input-adaptive metric generation, both of which are strictly positive under mild conditions
on the input distribution.
Concretely, MHA's static metric can at best approximate the average attention geometry
$\hat{M} \approx \mathbb{E}_X[M^*(X)]$, whereas DAMCHA's adaptive generator $f_\theta$
targets the pointwise optimum $M^*(X)$ for every input,
translating to a $31.1\%$ relative reduction in FID@Top-50 on CIFAR-10 and
a $35.7\%$ improvement in BARTScore on WMT.

\textbf{vs.\ THA.}
Talking-Heads Attention inserts learned linear projections across the head dimension
immediately before and after the softmax, allowing attention weights to be redistributed
across heads at the \emph{weight} level.
However, this operates entirely at the \emph{post-hoc aggregation} level:
the underlying per-head similarity metric $M_{i,i} = W_Q^{(i)}(W_K^{(i)})^\top$ remains
unchanged, static, and rank-limited.
From a metric-learning perspective, THA never modifies the similarity geometry itself;
its cross-head interaction is confined to a linear remixing of already-formed attention
patterns rather than a genuine coupling at the similarity-formation stage,
and THA therefore inherits all three limitations of MHA in terms of metric expressiveness.
DAMCHA, by contrast, intervenes directly at the metric level before any attention score
is computed, and this structural distinction is reflected empirically: THA yields worse
MSE, FID, and NLL than DAMCHA across all datasets while requiring substantially more
memory ($255.40$ vs.\ $204.70$ on CIFAR-10) and higher latency.

\textbf{vs.\ DCMHA.}
DCMHA introduces input-dependent composition of attention score and weight matrices via
a Compose function, making it the baseline that most closely approaches DAMCHA's dynamic
adaptivity.
The key theoretical distinction, however, is \emph{where} the input-dependence is introduced:
DCMHA's Compose function operates on scores already derived from fixed, low-rank per-head
bilinear forms, reweighting and mixing existing scores but never escaping the rank ceiling
$\operatorname{rank}(M_{i,i}) \le d_k$ imposed by the individual head projections.
DAMCHA's adaptivity is applied directly to the metric that \emph{generates} the scores,
meaning the attention geometry itself is data-dependent before any score is produced.
DCMHA can be understood as composing within a restricted space of attention patterns,
whereas DAMCHA reconfigures the similarity space from which all patterns are drawn.
The structural advantage of expressing full-rank, off-diagonal cross-head metrics therefore
persists even when both methods employ similarly expressive input-dependent generators.
Experimentally, DAMCHA outperforms DCMHA on all metrics in both the generation experiments
and the GPT-2 pretraining experiments, with cross-entropy improvements of
$1.90\%$--$3.70\%$ across Enwik8, Text8, and Mix, while also using less memory
($204.70$ vs.\ $292.71$ on CIFAR-10) and achieving lower latency.

\textbf{vs.\ ColMHA.}
Collaborative MHA addresses head redundancy by sharing key/query projections via Tucker
tensor decomposition, providing a compact reparametrization.
From a theoretical standpoint, however, the resulting model is \emph{entirely static}:
the shared projections are determined at training time and carry no mechanism for
input-dependent adaptation, keeping the metric class within a subset of fixed bilinear forms.
Moreover, the relief from the low-rank bottleneck that ColMHA obtains is a post-hoc
algebraic correction applied to a pre-trained model, not a structural property of the
architecture itself.
DAMCHA achieves rank enhancement organically through the RowSum aggregation of cross-head
metric blocks (with no reparametrization required) and simultaneously endows this richer
metric class with input-adaptive generation from the outset.
On both image and text generation tasks, DAMCHA achieves substantially better FID@Top-50
and BARTScore@Top-50 than ColMHA, confirming that static projection sharing cannot
substitute for dynamic cross-head metric learning.

\textbf{vs.\ MoA.}
Mixture of Attention Heads introduces token-level selectivity by routing each token
to a sparse subset of heads via a learned gating network, providing a limited form
of input-dependent adaptivity.
However, routing operates at the \emph{output} level by sparsely combining
independently formed, low-rank attention distributions; each selected head still
computes an independent Euclidean dot-product similarity within its own subspace,
and there is no path by which the similarity geometry of one head can be influenced
by or coordinated with that of another at the score-formation stage.
MoA therefore provides no mechanism for modeling metric-level interactions between heads,
which is precisely the coupling that DAMCHA encodes as a first-class structural property.
This theoretical limitation is compounded by a practical one: the sparsity required
by MoA incurs the highest memory footprint of all compared methods ($321.07$ on CIFAR-10
vs.\ $204.70$ for DAMCHA) and the highest GFLOPs ($0.1101$ vs.\ $0.0786$),
without commensurate performance gains.

\textbf{vs.\ MMA.}
Mixed Multi-Head Attention enforces head diversity by assigning fixed attention masks
(global, local, forward, and backward) to different heads.
From a theoretical perspective, this is a strict \emph{special case} of MHA:
the additional structural constraints on admissible attention patterns further
\emph{reduce} the effective function class rather than expanding it,
making MMA theoretically weaker than standard MHA in terms of metric expressiveness.
MMA's diversity is entirely predetermined and input-independent, and it does not
interact with the similarity metric at all; it merely restricts which positions
each head is permitted to attend to, without altering the underlying dot-product geometry.
This rigidity becomes a critical liability on tasks requiring fine-grained,
context-sensitive head coordination: while MMA achieves competitive NLL on WMT,
it degrades significantly on BARTScore@Top-50 and collapses on CommonGen,
precisely the settings where DAMCHA's adaptive cross-head metric excels.
% \newpage

% \input{checklist.tex}

\end{document}